\documentclass[letterpaper]{article} 
\usepackage[preprint]{aaai2027}  
\usepackage[hyphens]{url}  
\usepackage{graphicx} 
\usepackage{natbib}  
\usepackage{caption} 
\usepackage{algorithm}
\usepackage{algpseudocode}

\usepackage{newfloat}
\usepackage{listings}
\DeclareCaptionStyle{ruled}{labelfont=normalfont,labelsep=colon,strut=off} 
\floatstyle{ruled}
\newfloat{listing}{tb}{lst}{}
\floatname{listing}{Listing}

\usepackage{booktabs}
\usepackage{array}
\usepackage{tabularx}

\renewcommand{\topfraction}{0.90}
\renewcommand{\dbltopfraction}{0.90}
\renewcommand{\textfraction}{0.08}
\renewcommand{\floatpagefraction}{0.80}
\renewcommand{\dblfloatpagefraction}{0.80}

\usepackage{amsmath} 
\usepackage{graphicx}
\usepackage{subcaption}
\usepackage{amssymb}
\usepackage{amsmath}
\usepackage[most]{tcolorbox}

\usepackage{amsmath,amssymb,amsthm}
\newtheorem{assumption}{Assumption}

\newtheorem{theorem}{Theorem}
\newtheorem{corollary}[theorem]{Corollary}
\newtheorem{proposition}[theorem]{Proposition}
\usepackage{pifont}

\title{Mechanisms of Width Scaling in Normalized Residual Networks: The Effective Alignment Dimension}
\author{
    Jinhao Zhang\textsuperscript{\rm 1},
    Zeyu Liu\textsuperscript{\rm 2},
    Zicheng Yan\textsuperscript{\rm 3},
    Yunquan Zhang\textsuperscript{\rm 2},\\
    Guangming Tan\textsuperscript{\rm 2},
    Fangming Liu\textsuperscript{\rm 4},
    Daning Cheng\textsuperscript{\rm 2}
}
\affiliations{
    \textsuperscript{\rm 1}Beijing University of Posts and Telecommunications\\
    \textsuperscript{\rm 2}Institute of Computing Technology, Chinese Academy of Sciences\\
    \textsuperscript{\rm 3}University of Science and Technology of China\\
    \textsuperscript{\rm 4}Huazhong University of Science and Technology
}

\begin{document}

\maketitle

\begin{abstract}
Existing theories of neural-network width characterize asymptotic limits, but provide limited guidance on whether an expansion direction identified from finite training data remains beneficial on unseen data. We study this problem for function-preserving residual expansion and introduce the effective alignment dimension, a measurable quantity describing the signal--noise geometry of activation gradients. By deriving the exact mean and variance of the inner product between independently estimated training and test gradients, we obtain a finite-sample upper bound on misalignment probability. The bound depends only on the effective alignment dimension and an effective sample size, requiring finite second moments and a nonzero population gradient, without covariance spectral assumptions or prescribed width-growth rates. We integrate this certificate into the train--test residual-expansion framework, yielding a high-probability condition for test-risk improvement. Experiments across width-controlled LLaMA-style Transformers, Pythia, and ResNet-20 show that wider models exhibit larger effective alignment dimensions and lower empirical misalignment. Direct residual interventions confirm that the alignment statistic predicts the sign and magnitude of held-out loss changes.
\end{abstract}

\section{Introduction}
Width is one of the most thoroughly studied dimensions of neural scaling. A large body of theory examines large-width limits and finite-width corrections through Gaussian processes, kernel dynamics, feature-learning limits, and width-dependent fluctuations. These theories do not, however, directly provide a finite-sample certificate for whether a local expansion direction identified in a specific trained model at a specific insertion point will transfer to held-out data. This question arises naturally in model growth, where progressive stacking, Net2Net-style expansion, and learned growth operators must determine whether a proposed insertion remains beneficial beyond the finite training set.

Recent work on normalized residual networks \cite{cheng2026qualitative} addresses this problem by decomposing depth growth into representational gain, optimization gain, and generalization transfer. It establishes a qualitative mechanism: a residual block inserted at a function-preserving zero-output initialization creates a first-order descent direction whenever its features are not orthogonal to the insertion-point error signal, while width improves the finite-sample observability of that direction. We study how reliably a training-identified direction remains favorable on independent test data.

The remaining gap is quantitative finite-sample observability. The prior width-dependent alignment result requires control of the covariance spectrum, total variance, and growth of the mean gradient signal. These assumptions yield a clean asymptotic width rate, but not an instance-specific certificate for a given model and training--test pair. Moreover, the separate worst-case conditions do not expose the relative signal--noise geometry determining finite-sample directional stability.

We retain the activation-space formulation of the prior framework. For each example, the sample-level signal is the loss gradient with respect to the hidden representation at the insertion point. Averaging over the training and test sets produces two empirical activation-gradient directions whose alignment measures whether the training-identified direction remains favorable on independent test data. Positive alignment supports the same local residual direction on both samples, whereas nonpositive alignment marks a failure of finite-sample directional transfer. Through the function-preserving jumpboard construction, this event enters the comparison between the original and expanded models.

Rather than controlling training and test deviations separately, we analyze their alignment statistic directly. An exact second-moment characterization yields an instance-specific finite-sample certificate summarized by a dimensionless effective alignment dimension. Together with an effective sample size determined by the training and test set sizes, it controls the probability of nonpositive alignment, with a stronger certificate as their product increases.

The result requires only a finite second moment and a nonzero population mean for the sample-level activation gradient. It requires no covariance spectral bound, prescribed width-growth rate, or fourth-moment condition, and makes no architectural assumption on the inserted module or backbone. Under the covariance-spectrum, total-variance, and signal-growth conditions of prior work, it recovers the same-order width dependence. These assumptions thus become sufficient conditions for linear growth of the effective alignment dimension with width, rather than prerequisites for the guarantee. The dimension and its signal and noise quantities can be estimated from a trained model with a small number of forward--backward passes, without explicitly forming the covariance matrix.

We then integrate the certificate into the existing direct train--test expansion framework. Under the original assumptions ensuring an improving jumpboard model, realizability of the residual direction, subsequent optimization gain, and uniform generalization control, the resulting theorem replaces the previous alignment-control step with an instance-specific guarantee in terms of the training and test set sizes and the effective alignment dimension, while leaving the representational, optimization, and generalization components unchanged. The probability of successful directional transfer is therefore expressed directly through these quantities.

This analysis also clarifies the role of width. Width does not enter the certificate independently, but changes the activation-gradient distribution and hence the relative signal--noise geometry summarized by the effective alignment dimension. If the relative-noise structure does not deteriorate with width, increasing width cannot worsen the certificate; if the effective alignment dimension increases, the alignment-failure bound decreases. The benefit of width is therefore conditional rather than automatic and can be evaluated on the model at hand.

Experiments on width-controlled LLaMA-style Transformers, Pythia, and ResNet-20 support this account: $d_{\mathrm{align}}$ generally increases with width while empirical misalignment decreases under fixed sampling budgets. Varying the training and test sample sizes further supports the predicted finite-sample dependence. Direct residual interventions verify that positive alignment predicts the sign and local magnitude of realizable held-out loss changes.

\paragraph{Contributions.}
\ding{182} We derive an exact second-moment characterization of the alignment between training- and test-averaged activation gradients, yielding an instance-specific finite-sample certificate summarized by the effective alignment dimension.
\ding{183} We replace the previous width-dependent alignment step in the direct train--test expansion framework with this certificate, obtaining a joint high-probability guarantee in terms of training size, test size, and effective alignment dimension.
\ding{184} Across width-controlled Transformers and ResNets, we show that the effective alignment dimension generally increases with width alongside reduced empirical misalignment. Direct residual interventions further verify that the alignment statistic predicts both the sign and local magnitude of held-out loss changes.

\section{Related Work}

Theoretical studies of neural-network width primarily characterize
large-width limits and deviations around them. NNGP and NTK limits
yield Gaussian-process and kernel descriptions
\cite{neal1996bayesian,lee2018deep,jacot2018ntk}, while mean-field and
maximal-update parameterizations retain nontrivial feature learning
\cite{mei2018meanfield,chizat2018global,
yang2021tensorprogramsiv,yang2022tensorprogramsv}. Finite-width
analyses quantify width-dependent fluctuations of kernels and
predictions \cite{bordelon2023dynamics}, whereas neural scaling-law
theories describe average loss as data or model size grows
\cite{bahri2024explaining}. A separate literature studies coherent
gradients, parameter-gradient signal-to-noise ratios, and gradient
noise scale as diagnostics of optimization and generalization
\cite{chatterjee2020coherent,liu2020gsnr,
mccandlish2018empirical}, but does not analyze the two-sample
sign-reversal probability of an update selected for a newly inserted
module. Function-preserving model-growth methods, including Net2Net,
progressive stacking, bert2BERT, and LiGO, reuse trained models to
initialize larger ones
\cite{chen2016net2net,gong2019efficient,
chen2022bert2bert,wang2023ligo}; concurrent work emphasizes the
geometry of function preservation and post-insertion plasticity
\cite{lok2026gatezero,khemais2026exact}. Closest to our setting,
Cheng et al.~\cite{cheng2026qualitative} decompose residual expansion
into representational, optimization, and generalization components
and derive a width-dependent train/test activation-gradient alignment guarantee.

\section{Assumptions and Notation}
\label{sec:assumptions-notation}

We follow the activation-space notation and the direct
train--test comparison framework of
\citet{cheng2026qualitative}. To match the notation used in
our theoretical analysis and experiments, we denote the
activation-gradient signal at the insertion point by $q$.
This is the same activation-gradient object denoted by
$\zeta$ in the train--test alignment theorem of the prior
framework.

Throughout the paper, we fix a well-trained reference model $f_{\mathrm{old}}^{*}$, and $f_{\mathrm{old}} =
f_{\mathrm{top}}
\circ
f_{\mathrm{bot}}\in \mathcal H_{\mathrm{old}}$ and a candidate insertion point. For an input $x$, let $z(x)
:=
f_{\mathrm{bot}}(x)
\in
\mathbb R^{N}$
denote the hidden representation at that point, where $N$
is its width. Let $\mathcal H_{\mathrm{new}}$ denote the hypothesis class
obtained by inserting a residual function at the selected location.  For any predictor $f$ and loss function $\ell$, define the population, training, and
test risks by $R(f):=
\mathbb E_{(x,y)\sim\mathcal D}
\left[
\ell\bigl(f(x),y\bigr)
\right]$, $\mathcal L_{\mathrm{train}}(f)=
\frac{1}{M}
\sum_{i=1}^{M}
\ell\bigl(f(x_i),y_i\bigr)$,
$\mathcal L_{\mathrm{test}}(f)
:=
\frac{1}{K}
\sum_{j=1}^{K}
\ell\bigl(f(\widetilde x_j),\widetilde y_j\bigr).
$ 

\subsection{Basic Assumptions}
\label{subsec:basic-assumptions}

\begin{assumption}[Independent sampling and second moments]
\label{ass:alignment}

Let $ 
\mathcal S_{\mathrm{train}}
=
\left\{
(x_i,y_i)
\right\}_{i=1}^{M}$, $
\mathcal S_{\mathrm{test}}
=
\left\{
(\widetilde x_j,\widetilde y_j)
\right\}_{j=1}^{K}
$ be independent samples satisfying $
\mathcal S_{\mathrm{train}}
\overset{\mathrm{i.i.d.}}{\sim}
\mathcal D^{M}$, 
$\mathcal S_{\mathrm{test}}
\overset{\mathrm{i.i.d.}}{\sim}
\mathcal D^{K}$, 
$\mathcal S_{\mathrm{train}}
\perp
\mathcal S_{\mathrm{test}}$.

All alignment probabilities are understood conditional on
the fixed reference model $f_{\mathrm{old}}^{*}$ and the fixed
insertion point.

For a sample $(x,y)$, define the activation-gradient signal
at the insertion point by $
q(x,y)
:=
\left.
\nabla_z
\ell\bigl(f_{\mathrm{top}}(z),y\bigr)
\right|_{z=z(x)}
\in
\mathbb R^{N}$.

Define its population mean and covariance by $
\bar\mu
:=
\mathbb E_{(x,y)\sim\mathcal D}
\left[
q(x,y)
\right]$, 
$\Sigma
:=
\operatorname{Cov}_{(x,y)\sim\mathcal D}
\left(
q(x,y)
\right)$. We assume $
\mathbb E_{(x,y)\sim\mathcal D}
\left[
\left\|
q(x,y)
\right\|_2^2
\right]
<
\infty,
\bar\mu
\neq
0$.
\end{assumption}

Assumption~\ref{ass:alignment} is the only assumption
required for Theorem~\ref{thm:alignment}. In particular, the
finite-sample alignment result requires no covariance
spectral bound, fourth-moment condition, or prescribed
width-growth rate.

\begin{assumption}[Selection condition]
\label{ass:selection}

Let $
\widetilde f_S
\in
\mathcal H_{\mathrm{new}}$ denote the empirical jumpboard model constructed from
$\mathcal S_{\mathrm{train}}$ in the direct train--test
expansion framework, and let $
f_{\mathrm{new}}
\in
\mathcal H_{\mathrm{new}}$
denote the final expanded model.

We assume $\mathcal L_{\mathrm{train}}
\left(
f_{\mathrm{new}}
\right)
\leq
\mathcal L_{\mathrm{train}}
\left(
\widetilde f_S
\right)$. Equivalently, define the nonnegative optimization gain by $\Delta_{\mathrm{ERM}}
:=
\mathcal L_{\mathrm{train}}
\left(
\widetilde f_S
\right)
-
\mathcal L_{\mathrm{train}}
\left(
f_{\mathrm{new}}
\right)
\geq
0$.
\end{assumption}

\begin{assumption}[Uniform generalization control]
\label{ass:generalization}

Fix a confidence parameter $\delta\in(0,1)$. We assume that
there exist deterministic generalization radii
$\epsilon_M,\epsilon_K\geq0$ such that
$
\Pr
\left[
\sup_{f\in\mathcal H_{\mathrm{new}}}
\left|
R(f)
-
\mathcal L_{\mathrm{train}}(f)
\right|
\leq
\epsilon_M
\right]
\geq
1-\delta$
and
$\Pr
\left[
\sup_{f\in\mathcal H_{\mathrm{new}}}
\left|
R(f)
-
\mathcal L_{\mathrm{test}}(f)
\right|
\leq
\epsilon_K
\right]
\geq
1-\delta$.
\end{assumption}

The boundedness, Lipschitz, normalization, and norm-control
conditions used by \citet{cheng2026qualitative} to instantiate
$\epsilon_M$ and $\epsilon_K$ are restated in the appendix.
The main text uses only the two uniform inequalities in
Assumption~\ref{ass:generalization}.

For the direct train--test comparison, define the finite-test
jumpboard margin by $\Delta_R^{\mathrm{test}}
:=
\mathcal L_{\mathrm{test}}
\left(
f_{\mathrm{old}}^{*}
\right)
-
\mathcal L_{\mathrm{test}}
\left(
\widetilde f_S
\right)$.

Despite the subscript $R$ inherited from the prior notation,
$\Delta_R^{\mathrm{test}}$ denotes a finite-test empirical
margin rather than a population-risk difference.

The small-step and residual-realizability conditions used to
convert positive activation-gradient alignment into a
positive finite-test jumpboard margin are invoked directly
from \citet[Corollary~3]{cheng2026qualitative} and restated
in the appendix. They are not assumptions of
Theorem~\ref{thm:alignment} and are therefore not repeated
as global assumptions here.

\subsection{Key Alignment Quantities}
\label{subsec:key-quantities-certificates}
Define the empirical activation-gradient averages over the
training and test sets by $
\mu_M
:=
\frac{1}{M}
\sum_{i=1}^{M}
q(x_i,y_i)$, $g_K
:=
\frac{1}{K}
\sum_{j=1}^{K}
q(\widetilde x_j,\widetilde y_j)$. The finite-sample activation-gradient misalignment event is $\left\{
\mu_M^{\top}g_K
\leq
0
\right\}$. 

Define the directional effective dimension by $d_{\parallel}
:=
\frac{
\left\|
\bar\mu
\right\|_2^4
}{
\bar\mu^{\top}
\Sigma
\bar\mu
}$,
and the covariance-energy effective dimension by $
d_2
:=
\frac{
\left\|
\bar\mu
\right\|_2^4
}{
\operatorname{tr}(\Sigma^2)
}
=
\frac{
\left\|
\bar\mu
\right\|_2^4
}{
\left\|
\Sigma
\right\|_F^2
}$.

If a denominator in either definition is zero, the
corresponding effective dimension is defined as $+\infty$. The effective alignment dimension is
\begin{small}
\begin{equation}
    d_{\mathrm{align}}
:=
\min
\left\{
d_{\parallel},
d_2
\right\}
\end{equation}.
\end{small}

The exact finite-sample certificate uses $\Xi_{M,K}
:=
\frac{1}{M d_{\parallel}}
+
\frac{1}{K d_{\parallel}}
+
\frac{1}{MK d_2}$,
and $B_{\mathrm{exact}}(M,K)
:=
\frac{
\Xi_{M,K}
}{
1+\Xi_{M,K}
}$. 

Define the effective sample-size factor by
\begin{small}
\begin{equation}
n_{\mathrm{eff}}(M,K)
:=
\frac{
MK
}{
M+K+1
},
\end{equation}
\end{small}
and the simplified effective-alignment certificate by
\begin{small}
\begin{equation}
B_{\mathrm{align}}(M,K)
:=
\frac{
1
}{
1+
n_{\mathrm{eff}}(M,K)
d_{\mathrm{align}}
}.
\end{equation}
\end{small}
The two certificates satisfy
\begin{equation*}
B_{\mathrm{exact}}(M,K)
\leq
B_{\mathrm{align}}(M,K)
<
1.
\end{equation*}

When dependence on the insertion width must be made
explicit, we write
\begin{equation*}
\bar\mu^{(N)},
\qquad
\Sigma^{(N)},
\qquad
d_{\parallel}(N),
\qquad
d_2(N),
\qquad
d_{\mathrm{align}}(N).
\end{equation*}
At a fixed insertion width, the width argument is omitted.

Theorem~\ref{thm:alignment} uses only
Assumption~\ref{ass:alignment}. The core risk comparison in
Theorem~\ref{thm:direct-transfer} uses
Assumptions~\ref{ass:selection} and
\ref{ass:generalization}. Its additional positive-margin
statement combines Theorem~\ref{thm:alignment} with the
conditions of \citet[Corollary~3]{cheng2026qualitative},
which are restated in the appendix.

\section{Main Results}
\label{sec:main-results}

\subsection{Main Theorems}
\label{subsec:main-theorems}

Our first result controls the probability that the activation-gradient
direction estimated from the training set fails to remain positively
aligned on the independent test set. In addition to the probability
bound, we state the exact mean and variance identities that generate
the two effective dimensions $d_{\parallel}$ and $d_{2}$.

\begin{theorem}[Finite-sample activation-gradient alignment]
\label{thm:alignment}

Under Assumption~\ref{ass:alignment}, the train--test alignment
statistic satisfies $\mathbb E
\left[
\mu_M^{\top}g_K
\right]
=
\left\|
\bar\mu
\right\|_2^2 $
and
$
\operatorname{Var}
\left(
\mu_M^{\top}g_K
\right)
=
\left(
\frac{1}{M}
+
\frac{1}{K}
\right)
\bar\mu^{\top}
\Sigma
\bar\mu
+
\frac{1}{MK}
\operatorname{tr}
\left(
\Sigma^2
\right).
$

Consequently,
\begin{equation*}
\Pr
\left(
\mu_M^{\top}g_K
\leq
0
\right)
\leq
B_{\mathrm{exact}}(M,K)
=
\frac{
\Xi_{M,K}
}{
1+\Xi_{M,K}
},
\end{equation*}

Moreover, the one-scalar certificate satisfies
\begin{align*}
\Pr
\left(
\mu_M^{\top}g_K
\leq
0
\right)
\leq
B_{\mathrm{exact}}(M,K)\\
\leq
B_{\mathrm{align}}(M,K)
=
\frac{
1
}{
1+
n_{\mathrm{eff}}(M,K)
d_{\mathrm{align}}
}
\end{align*}
\end{theorem}

The exact certificate retains the two relative-noise contributions
separately. The simplified certificate replaces them by their
conservative minimum, the effective alignment dimension. The theorem
requires neither a covariance spectral bound nor a prescribed
width-growth law: those conditions are unnecessary for the
fixed-width finite-sample guarantee.

The next result combines the new alignment certificate with the
direct train--test comparison framework. Its first part is the
algebraic comparison between the original model, the empirical
jumpboard, and the final expanded model. Its second part replaces the
previous width-dependent alignment bound by
Theorem~\ref{thm:alignment}.

\begin{theorem}[Direct train--test expansion with
effective-alignment control]
\label{thm:direct-transfer}

Under Assumptions~\ref{ass:selection} and
\ref{ass:generalization}, with probability at least $
1-2\delta$,
the final expanded model satisfies
\begin{equation*}
\mathcal L_{\mathrm{test}}
\left(
f_{\mathrm{new}}
\right)
\leq
\mathcal L_{\mathrm{test}}
\left(
f_{\mathrm{old}}^{*}
\right)
-
\Delta_R^{\mathrm{test}}
-
\Delta_{\mathrm{ERM}}
+
2
\left(
\epsilon_M+\epsilon_K
\right).
\end{equation*}

Consequently, on every realization for which
\begin{equation*}
\Delta_R^{\mathrm{test}}
+
\Delta_{\mathrm{ERM}}
>
2
\left(
\epsilon_M+\epsilon_K
\right),
\end{equation*}
one has the strict test-risk improvement
\begin{equation*}
\mathcal L_{\mathrm{test}}
\left(
f_{\mathrm{new}}
\right)
<
\mathcal L_{\mathrm{test}}
\left(
f_{\mathrm{old}}^{*}
\right).
\end{equation*}

Suppose additionally that the empirical jumpboard is chosen in the
small-step regime and satisfies the residual-realizability conditions
of \citet[Corollary~3]{cheng2026qualitative}. Then positive
activation-gradient alignment implies a positive finite-test
jumpboard margin. Hence, by Theorem~\ref{thm:alignment},
\begin{equation*}
\Pr
\left(
\Delta_R^{\mathrm{test}}
>
0
\right)
\geq
1-
B_{\mathrm{exact}}(M,K)
\geq
1-
B_{\mathrm{align}}(M,K).
\end{equation*}

The two uniform-generalization events and the positive-margin event
therefore hold simultaneously with probability at least $ 1
-
2\delta
-
B_{\mathrm{exact}}(M,K)$, 
and using the simplified certificate, with probability at least
$
1
-
2\delta
-
B_{\mathrm{align}}(M,K).
$
\end{theorem}

The first part of Theorem~\ref{thm:direct-transfer} does not require
the alignment event and remains valid even when the finite-test margin
is zero or negative. The alignment-supported part instead quantifies
how often this random margin is positive. Importantly,
Theorem~\ref{thm:alignment} controls the sign of the margin through
the inherited jumpboard bridge; it does not lower-bound the magnitude
of the margin. Strict improvement therefore continues to require the
realized margin and optimization gain to dominate the two
generalization radii.

\subsection{The Role of Width}
\label{subsec:role-of-width}

Theorem~\ref{thm:alignment} is a fixed-width finite-sample
result. To compare models of different widths, we make the
width dependence of the activation-gradient distribution and
the effective alignment dimension explicit. The simplified
certificate becomes $B_{\mathrm{align}}(M,K;N)$.

This factorization separates the effects of data and model
width. The effective sample-size factor depends only on the
training and test set sizes. Width instead changes the
distribution of the activation gradient, and hence its
population mean, covariance, and effective alignment
dimension. Width therefore does not improve the certificate
directly; it does so only when it improves the relative
signal--noise geometry summarized by
$d_{\mathrm{align}}(N)$.

A sufficient condition for a wider model to have a
non-worse simplified certificate is that neither component of
the effective alignment dimension deteriorates. In
particular, if $N_2>N_1$,  $d_{\parallel}(N_2)
\geq
d_{\parallel}(N_1)$,
$d_2(N_2)
\geq
d_2(N_1)$
$\Longrightarrow$ 
$B_{\mathrm{align}}(M,K;N_2)
\leq
B_{\mathrm{align}}(M,K;N_1).
$ This is a sufficient rather than necessary condition: the
exact certificate in Theorem~\ref{thm:alignment} retains the
two relative-noise contributions separately and may give a
finer comparison.

The stronger covariance and signal-growth conditions used
in the prior width-dependent analysis provide one sufficient
route to this behavior. As shown in the appendix, those
conditions imply, for sufficiently large $N$, $d_{\parallel}(N)
=
\Omega(N)$,
$d_2(N)
=
\Omega(N)$,
$d_{\mathrm{align}}(N)
=
\Omega(N)$.
Consequently, for a width-independent constant $c_0>0$,
\begin{equation*}
\begin{aligned}
\Pr\left(
\mu_M^{\top}g_K
\leq
0
\right)
&\leq
\frac{1}{
1+
c_0
n_{\mathrm{eff}}(M,K)
N
}
\\
&=
O\left(
\frac{1}{MN}
+
\frac{1}{KN}
+
\frac{1}{MKN}
\right).
\end{aligned}
\end{equation*}

Up to the lower-order mixed term, this recovers the
width dependence of the prior alignment result. The
difference is that these growth conditions are not required
by Theorem~\ref{thm:alignment}; they are only sufficient
conditions for linear growth of the effective alignment
dimension. More generally, the certificate improves at
whatever rate is supported by the observed behavior of
$d_{\mathrm{align}}(N)$. Thus, the benefit of width for
finite-sample directional transfer is conditional rather than
automatic and can be evaluated directly on the model at hand.

\begin{figure*}[!t]
    \centering
    \begin{subfigure}[t]{0.45\textwidth}
        \centering
        \includegraphics[width=\linewidth]{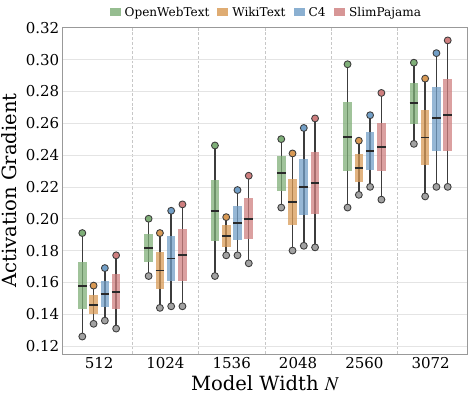}
        \caption{Activation-gradient statistic across model widths.}
        \label{fig:activation-gradient}
    \end{subfigure}
    \hfill
    \begin{subfigure}[t]{0.45\textwidth}
        \centering
        \includegraphics[width=\linewidth]{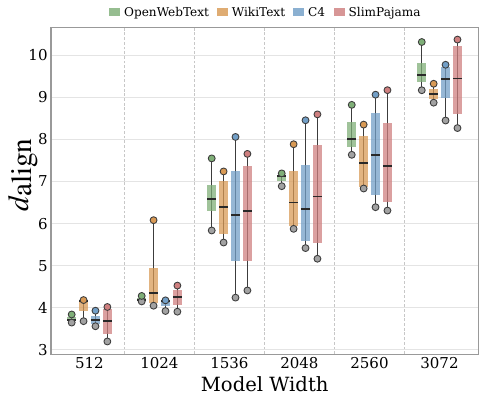}
        \caption{Effective alignment dimension across model widths.}
        \label{fig:d-align}
    \end{subfigure}
    \caption{Width-dependent activation-gradient geometry. Panel (a) reports the activation-gradient statistic $\widehat{S}=\|\widehat{\mu}\|_2^2$, whereas panel (b) reports $d_{\mathrm{align}}$ computed from the same activation-gradient samples. For each width, metric estimates are first aggregated across independently sampled evaluation subsets within each trained model and then summarized across eight training seeds.}
    \label{fig:gradient-alignment}
\end{figure*}

\section{Experiment}
Our experiments evaluate four questions: how activation-gradient geometry scales with width, whether empirical misalignment follows the predicted width and sample-size dependence, whether the trends transfer across Pythia and ResNet, and whether positive alignment predicts the loss change of a realizable residual intervention.

\subsection{Experiment Configurations}
\paragraph{Model.}We consider three experimental settings: our pretrained 32-layer LLaMA-style Transformer models \cite{Touvron2023LLaMAOA}, the Pythia model suite, and independently trained ResNet models. For the LLaMA-style family, we pretrain six models with hidden widths $N\in\{512,1024,1536,2048,2560,3072\}$ while keeping the depth, training data, token budget, and optimization settings fixed. For each LLaMA-style and ResNet configuration, we train eight independent models using different random seeds. We further evaluate Pythia checkpoints of different scales.

\paragraph{Metric Evaluation.}
For LLaMA-style and Pythia models, the activation gradient $q$ is measured at the final Transformer block, whereas for ResNet-20 it is measured at the first residual block. Given evaluation samples $\{q_r\}_{r=1}^{R_{\mathrm{eval}}}$, we estimate $\widehat{\mu}=R_{\mathrm{eval}}^{-1}\sum_{r=1}^{R_{\mathrm{eval}}}q_r$ and report the activation-gradient statistic $\widehat{S}=\|\widehat{\mu}\|_2^2$. We compute $d_{\mathrm{align}}$ and the empirical misalignment probability from the same activation-gradient samples, with the measurement location fixed across widths within each family. Metrics are estimated using multiple data-sampling seeds on C4, WikiText-2, OpenWebText, and SlimPajama for language models, and on CIFAR-10 and CIFAR-100 for ResNet-20. For independently trained models, data-sampling variability is first aggregated within each checkpoint and then summarized across training seeds. Full architectural configurations, training protocols, dataset preprocessing, and metric-estimation details are provided in the appendix.

\subsection{Width Scaling of Activation-Gradient Geometry}
\label{sec:width-scaling}
We first examine how activation-gradient geometry varies with model width in the controlled LLaMA-style family. We consider hidden widths $N\in\{512,1024,1536,2048,2560,3072\}$ while holding the model depth, training corpus, token budget, and optimization protocol fixed. Figure~\ref{fig:gradient-alignment} reports the activation-gradient statistic $\widehat{S}$ and the effective alignment dimension $d_{\mathrm{align}}$, both computed from the same activation-gradient samples, on C4, WikiText-2, OpenWebText, and SlimPajama. Detailed metric-estimation procedures and pseudocode are provided in the appendix.

As shown in Figure~\ref{fig:activation-gradient}, the activation-gradient statistic exhibits a clear width-dependent trend. Although its absolute scale varies across datasets and individual runs, the distributions generally shift upward as $N$ increases. This pattern persists across independently trained models and independently sampled evaluation subsets, indicating that the observed trend is robust to both optimization randomness and finite-sample estimation variability.

Figure~\ref{fig:d-align} shows a corresponding increase in the effective alignment dimension computed from the activation-gradient signal \(q\). Across the four datasets, $d_{\mathrm{align}}$ grows from approximately $3$--$4$ at the smallest widths to approximately $9$--$10$ at the largest width. Although adjacent widths exhibit overlapping distributions and moderate local fluctuations, the aggregated trend remains consistently increasing over the investigated width range.
Taken together, Figure~\ref{fig:gradient-alignment} provides empirical support for the central width-scaling hypothesis
\begin{tcolorbox}[
    colback=blue!4,
    colframe=blue!45!black,
    boxrule=0.6pt,
    arc=2pt,
    left=5pt,
    right=5pt,
    top=2pt,
    bottom=2pt,
    before skip=6pt,
    after skip=6pt
]
\begin{equation*}
\boldsymbol{
N \uparrow
\quad\Longrightarrow\quad
d_{\mathrm{align}}(N) \uparrow
}
\end{equation*}
\end{tcolorbox}

\begin{figure*}[!t]
    \centering
    \begin{subfigure}[t]{0.245\textwidth}
        \centering
        \includegraphics[width=\linewidth]{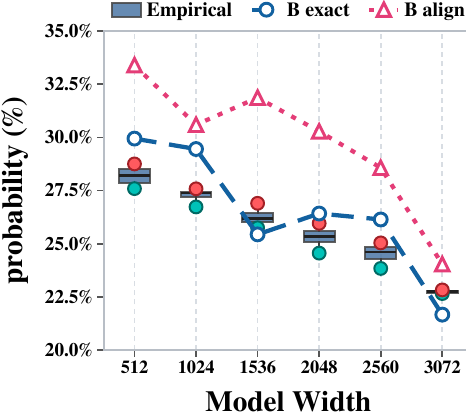}
        \caption{$M=K=1$.}
        \label{fig:p-mk1}
    \end{subfigure}
    \hfill
    \begin{subfigure}[t]{0.245\textwidth}
        \centering
        \includegraphics[width=\linewidth]{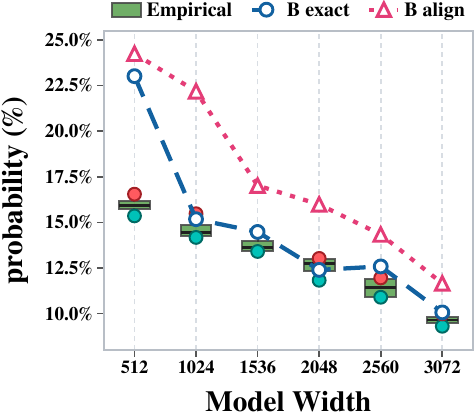}
        \caption{$M=K=2$.}
        \label{fig:p-mk2}
    \end{subfigure}
    \hfill
    \begin{subfigure}[t]{0.245\textwidth}
        \centering
        \includegraphics[width=\linewidth]{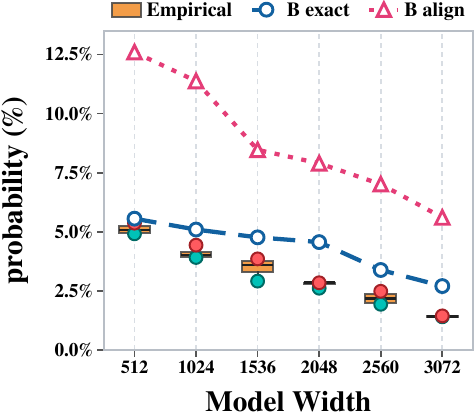}
        \caption{$M=K=4$.}
        \label{fig:p-mk4}
    \end{subfigure}
    \hfill
    \begin{subfigure}[t]{0.245\textwidth}
        \centering
        \includegraphics[width=\linewidth]{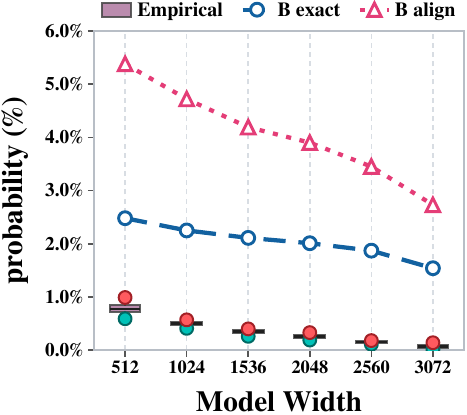}
        \caption{$M=K=8$.}
        \label{fig:p-mk8}
    \end{subfigure}

    \caption{Empirical gradient-misalignment probabilities and theoretical certificates across model widths and finite-sample budgets. Each sampled sequence contains $2048$ tokens, and each empirical probability is estimated using $R=500{,}000$ independent Monte Carlo trials. The exact certificate $B_{\mathrm{exact}}$ and simplified certificate $B_{\mathrm{align}}$ are computed from the corresponding activation-gradient moments. Results are summarized across eight independently trained LLaMA-style models and multiple independent data-sampling seeds.}
    \label{fig:misalignment-probability}
\end{figure*}

\subsection{Gradient-Misalignment Probability}
\label{sec:misalignment-probability}

We next directly test the theoretical prediction that finite-sample gradient misalignment becomes less likely as the effective alignment dimension increases. To isolate the effect of model width from variation in the evaluation distribution, we conduct this analysis exclusively on C4 using our trained LLaMA-style models. For each trained model and sampling budget $(M,K)$, we estimate the gradient-misalignment probability using the Monte Carlo estimator
\begin{equation}
\widehat{p}_{M,K}(N)
=
\frac{1}{R}
\sum_{r=1}^{R}
\mathbf{1}
\left\{
\widehat{\mu}_{M,r}^{\top}
\widehat{g}_{K,r}
\leq 0
\right\},
\label{eq:empirical-misalignment-probability}
\end{equation}
where $\widehat{\mu}_{M,r}$ and $\widehat{g}_{K,r}$ are computed from two independently sampled subsets containing $M$ and $K$ sequences, respectively, in Monte Carlo trial $r$. Each sampled sequence contains $2048$ tokens. Independently sampling the two subsets matches the independence condition in Theorem~1 and avoids introducing artificial correlation between the two finite-sample gradient estimates.

We evaluate $M=K\in\{1,2,4,8\}$ and perform $R=500{,}000$ independent Monte Carlo trials for each combination of model width, training seed, and sampling budget. We additionally report the exact second-moment certificate $B_{\mathrm{exact}}$ and the simplified effective-alignment certificate $B_{\mathrm{align}}$, computed from the estimated activation-gradient moments. In evaluations with $M=K>16$, the empirical misalignment probabilities are close to zero for nearly all investigated widths. We therefore omit these larger sampling budgets from the main figure to preserve visual resolution. Detailed sampling procedures, pseudocode, and Monte Carlo error analysis are provided in the appendix.

Figure~\ref{fig:misalignment-probability} shows that the empirical gradient-misalignment probability consistently decreases as model width increases, and that this trend holds for every fixed sampling budget. Because $M$ and $K$ remain constant within each panel, the observed reduction cannot be attributed to wider models receiving more evaluation data. Instead, wider models yield more statistically reliable finite-sample gradient directions under the same sampling budget.

Both theoretical certificates exhibit the same decreasing trend with width and remain conservative relative to the empirical probabilities. As expected from Theorem~1, $B_{\mathrm{exact}}$ is consistently tighter than $B_{\mathrm{align}}$, while the latter provides a simpler one-scalar characterization through $d_{\mathrm{align}}$. The agreement in trend between the empirical probabilities and the two certificates supports the predicted connection between activation-gradient geometry and finite-sample directional transfer.

For $M=K=1$, the empirical misalignment probability decreases from approximately $28.5\%$ at $N=512$ to approximately $22.7\%$ at $N=3072$. For $M=K=2$, it decreases from approximately $15.7\%$ to approximately $9.7\%$. At larger sampling budgets, the reduction becomes more pronounced relative to the initial probability: for $M=K=4$, the probability decreases from approximately $5.2\%$ to approximately $1.4\%$, whereas for $M=K=8$, it decreases from approximately $0.7\%$ to below $0.1\%$.

The results also demonstrate the complementary role of the finite-sample budget. At every fixed model width, increasing $M$ and $K$ substantially reduces the probability of estimating a misaligned direction. This reduction is particularly pronounced between $M=K=1$ and $M=K=8$, and the probability becomes nearly zero once $M=K$ exceeds $16$. Thus, model width and sample size improve directional reliability through distinct mechanisms: increasing width improves the underlying activation-gradient geometry, whereas increasing $M$ and $K$ reduces finite-sample estimation uncertainty.
Together with the width-dependent increase in $d_{\mathrm{align}}$ observed in Figure~\ref{fig:gradient-alignment}, these results provide empirical support for the mechanism
\begin{tcolorbox}[
    colback=blue!4,
    colframe=blue!45!black,
    boxrule=0.6pt,
    arc=2pt,
    left=5pt,
    right=5pt,
    top=1pt,
    bottom=2pt,
    before skip=6pt,
    after skip=6pt
]
\begin{equation*}
N \uparrow
\quad\Longrightarrow\quad
d_{\mathrm{align}}(N) \uparrow
\quad\Longrightarrow\quad
p_{M,K}(N) \downarrow 
\end{equation*}
\end{tcolorbox}

\subsection{Cross-Architecture and Cross-Modality Validation}
\label{sec:cross-architecture}
\subsubsection{Fixed-Depth Width Validation on Pythia}
\label{sec:pythia-validation}
We next evaluate two fixed-depth Pythia regimes to separate width from depth. The six-layer regime contains Pythia-14M, 31M, and 70M with widths $N\in\{128,256,512\}$ and ten training seeds per scale; the 32-layer regime contains Pythia-2.8B and 6.9B with widths $N\in\{2560,4096\}$. Both $\widehat{S}$ and $d_{\mathrm{align}}$ are computed from the same activation-gradient samples.

To disentangle the effect of width from that of depth, we separately analyze two fixed-depth Pythia regimes. The small-scale group consists of Pythia-14M, Pythia-31M, and Pythia-70M, all with six Transformer layers and hidden widths \(N \in \{128, 256, 512\}\). For each model scale, we evaluate ten independently trained runs corresponding to distinct training seeds, including the additional seed replicas released by PolyPythias~\cite{Wal2025PolyPythiasSA}. The large-scale group consists of Pythia-2.8B and Pythia-6.9B, both with 32 Transformer layers and hidden widths \(N \in \{2560, 4096\}\). Within each group, model depth is fixed while hidden width and the architectural dimensions conventionally coupled to width increase.

\begin{figure}[!t]
    \centering
    \begin{subfigure}[t]{0.495\linewidth}
        \centering
        \includegraphics[width=\linewidth]{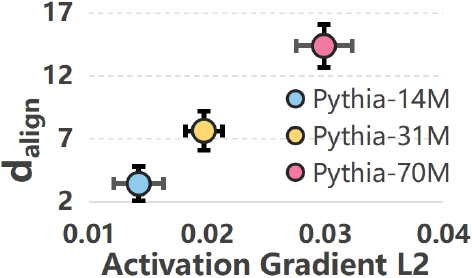}
        \caption {6-layer Pythia models.}
        \label{pythia-small}
    \end{subfigure}
    \hfill
    \begin{subfigure}[t]{0.495\linewidth}
        \centering
        \includegraphics[width=\linewidth]{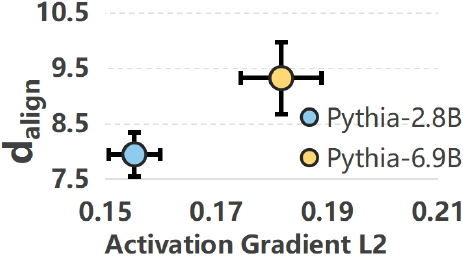}
        \caption {32-layer Pythia models.}
        \label{pythia-large}
    \end{subfigure}
    \caption{Controlled within-depth validation on Pythia. The six-layer group uses hidden widths \(N\in\{128,256,512\}\), with each configuration evaluated over ten independently trained runs. The 32-layer group uses hidden widths \(N\in\{2560,4096\}\), with one pretrained checkpoint evaluated at each width. Error bars summarize variability across training runs for the six-layer group and across independently sampled evaluation subsets for the 32-layer group.}
    \label{fig:pythia-validation}
\end{figure}

Figure~\ref{pythia-small} shows a clear monotonic progression within the six-layer regime. As the hidden width increases from $128$ to $512$, the activation-gradient statistic $\widehat S$ increases from approximately $0.014$ to $0.030$, while $d_{\mathrm{align}}$ increases from approximately $3$ to $14$. Thus, the three model configurations move consistently upward and to the right in the gradient-geometry plane, reproducing the width-dependent behavior observed in our controlled LLaMA-style models.

Figure~\ref{pythia-large} provides a complementary validation at substantially larger scale. With depth fixed at $32$ layers, increasing the hidden width from $2560$ in Pythia-2.8B to $4096$ in Pythia-6.9B increases both the activation-gradient statistic $\widehat{S}$ and $d_{\mathrm{align}}$.

The absolute values of $d_{\mathrm{align}}$ should not be compared directly across the two panels because the six-layer and 32-layer models differ in depth and overall architecture scale. The relevant result is the within-regime trend: at fixed depth, increasing hidden width consistently increases both $\widehat{S}$ and $d_{\mathrm{align}}$. This distinction prevents depth variation from being incorrectly attributed to width.

\subsubsection{Cross-Architecture Validation on ResNet}
\label{sec:resnet-validation}

We next test whether the observed width-dependent gradient geometry extends beyond autoregressive Transformers to convolutional vision models. We train ResNet-20 models~\cite{He2015DeepRL} from scratch on CIFAR-10 and CIFAR-100~\cite{Krizhevsky2009LearningML} with channel-width multipliers $\{0.5\times,1\times,2\times,4\times\}$. Within each dataset, network depth and all training conditions are held fixed, while only the channel dimensions and their associated width-dependent parameters are scaled. Each configuration is independently trained with eight random seeds, and the proposed metrics are evaluated using multiple data-sampling seeds. Metric estimates are first aggregated across data-sampling seeds within each trained model and are then summarized across the eight independent training seeds. Optimization hyperparameters and training schedules are provided in the appendix.

\begin{figure}[!t]
    \centering
    \includegraphics[width=\linewidth]{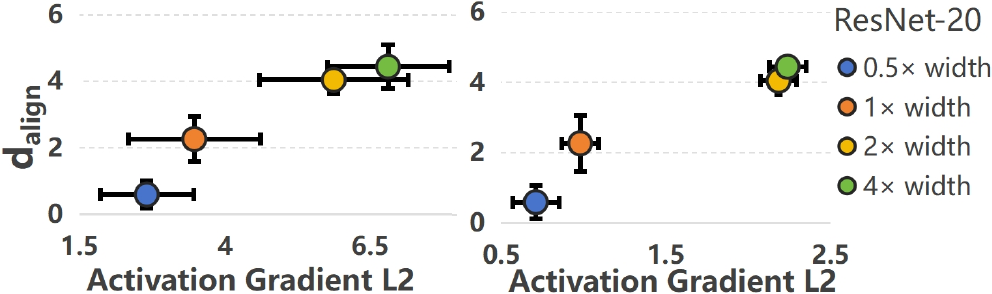}
    \caption{Controlled width validation on ResNet-20 for CIFAR-10 (left) and CIFAR-100 (right). The horizontal axis reports the activation-gradient statistic $\widehat{S}=\|\widehat{\mu}\|_2^2$, and the vertical axis reports $d_{\mathrm{align}}$. The channel-width multiplier is varied over $\{0.5\times,1\times,2\times,4\times\}$ while network depth and the training protocol are fixed. Error bars summarize variability across eight independently trained models after aggregation over multiple evaluation seeds.}
    \label{fig:resnet-validation}
\end{figure}
Figure~\ref{fig:resnet-validation} shows a consistent progression in activation-gradient geometry as channel width increases. On both CIFAR-10 and CIFAR-100, wider ResNet models move upward and to the right in the gradient-geometry plane, indicating increases in both the activation-gradient statistic $\widehat{S}$ and $d_{\mathrm{align}}$. This trend remains stable across independently trained models and evaluation samples despite differences in dataset complexity and metric scale.

Increasing the width multiplier from $0.5\times$ to $4\times$ produces substantial gains in both quantities on CIFAR-10, with the same qualitative trend observed on CIFAR-100. The smaller difference between the $2\times$ and $4\times$ settings indicates a weaker marginal increase over the investigated range.

\subsection{Direct Residual-Expansion Intervention}
\begin{table}[t]
\centering
\scriptsize
\setlength{\tabcolsep}{3.5pt}
\renewcommand{\arraystretch}{1.05}
\begin{tabular}{lccc}
\toprule
Model ($N$) & $\rho$ & Pred. $\Delta L$ & Obs. $\Delta L$ \\
\midrule
14M (128)
& $10^{-4}$ & $-0.307\pm0.002$  & $-0.305\pm0.005$ \\
& $10^{-3}$ & $-3.070\pm0.002$  & $-2.974\pm0.053$ \\
\midrule
31M (256)
& $10^{-4}$ & $-0.432\pm0.005$  & $-0.430\pm0.006$ \\
& $10^{-3}$ & $-4.316\pm0.003$  & $-4.171\pm0.063$ \\
\midrule
70M (512)
& $10^{-4}$ & $-1.058\pm0.002$  & $-1.057\pm0.013$ \\
& $10^{-3}$ & $-10.579\pm0.003$ & $-10.354\pm0.135$ \\
\bottomrule
\end{tabular}
\caption{Direct residual-intervention results. Observed loss changes are
reported as mean $\pm$ standard deviation across ten independently trained
checkpoints, after averaging $500$ trials within each checkpoint. All loss changes are in units of $10^{-3}$.}
\label{tab:residual-intervention}
\end{table}

We insert a zero-initialized token-shared residual parameter
$z_t(b)=z_t+b$ before the final normalization layer and evaluate
Pythia-14M, 31M, and 70M on C4 using ten independently trained
checkpoints per scale. For each checkpoint, we draw length-$2048$
sequences from two fixed disjoint pools of sizes $512$ and $256$,
with $M=K=8$ and $500$ trials. Results are averaged first within
each checkpoint and then across training seeds. The implementation
verifies $\nabla_b\ell=\sum_t\nabla_{z_t}\ell$ exactly, with zero
maximum relative $\ell_2$ error. Across all $15{,}000$ trials,
alignment is positive, the selected direction decreases held-out
loss, and its norm-matched reverse increases it.

Table~\ref{tab:residual-intervention} shows close agreement between
the predicted and observed loss changes across two representative
perturbation scales. At $\rho=10^{-4}$, the observed magnitudes retain
$99.3\%$--$99.9\%$ of their predicted values. Increasing $\rho$ to
$10^{-3}$ enlarges the loss reduction by approximately one order of
magnitude, while the observations retain $96.6\%$--$97.9\%$ of the
predicted magnitudes. The modest increase in prediction error at the
larger perturbation scale is consistent with higher-order terms becoming
more relevant away from the infinitesimal-step regime.

\section{Conclusion}
\label{sec:conclusion}
 
We studied when a local expansion direction identified from finite training data transfers to independent test data. Using the insertion-point activation gradient, we derived an exact finite-sample misalignment bound controlled by the effective sample size and $d_{\mathrm{align}}$. Width enters the certificate through the activation-gradient distribution: whenever $d_{\mathrm{align}}(N)$ increases with width, the misalignment bound decreases. Under the stronger covariance and signal-growth conditions of prior work, the certificate recovers the same-order width dependence. We incorporated this certificate into the direct train--test expansion framework. Experiments on width-controlled LLaMA-style Transformers, Pythia, and ResNet-20 show that $d_{\mathrm{align}}$ generally increases with width while empirical misalignment decreases under fixed sampling budgets. Direct residual interventions further confirm that positive alignment predicts the sign and local magnitude of held-out loss changes.

\bibliography{aaai2027}

\begingroup
\makeatletter
\@ifundefined{isChecklistMainFile}{
  \newif\ifreproStandalone
  \reproStandalonetrue
}{
  \newif\ifreproStandalone
  \reproStandalonefalse
}
\makeatother

\ifreproStandalone
\documentclass[letterpaper]{article}
\usepackage[submission]{aaai2027}
\setlength{\pdfpagewidth}{8.5in}
\setlength{\pdfpageheight}{11in}
\frenchspacing

\begin{document}
\fi
\setlength{\leftmargini}{20pt}
\makeatletter\def\@listi{\leftmargin\leftmargini \topsep .5em \parsep .5em \itemsep .5em}
\def\@listii{\leftmargin\leftmarginii \labelwidth\leftmarginii \advance\labelwidth-\labelsep \topsep .4em \parsep .4em \itemsep .4em}
\def\@listiii{\leftmargin\leftmarginiii \labelwidth\leftmarginiii \advance\labelwidth-\labelsep \topsep .4em \parsep .4em \itemsep .4em}\makeatother

\setcounter{secnumdepth}{0}
\renewcommand\thesubsection{\arabic{subsection}}
\renewcommand\labelenumi{\thesubsection.\arabic{enumi}}

\newcounter{checksubsection}
\newcounter{checkitem}[checksubsection]

\newcommand{\checksubsection}[1]{%
  \refstepcounter{checksubsection}%
  \paragraph{\arabic{checksubsection}. #1}%
  \setcounter{checkitem}{0}%
}

\newcommand{\checkitem}{%
  \refstepcounter{checkitem}%
  \item[\arabic{checksubsection}.\arabic{checkitem}.]%
}
\newcommand{\question}[2]{\normalcolor\checkitem #1 #2 \color{blue}}
\newcommand{\ifyespoints}[1]{\makebox[0pt][l]{\hspace{-15pt}\normalcolor #1}}

\section*{Reproducibility Checklist}

\vspace{1em}
\hrule
\vspace{1em}


\checksubsection{General Paper Structure}
\begin{itemize}

\question{Includes a conceptual outline and/or pseudocode description of AI methods introduced}{(yes/partial/no/NA)}
yes

\question{Clearly delineates statements that are opinions, hypothesis, and speculation from objective facts and results}{(yes/no)}
yes

\question{Provides well-marked pedagogical references for less-familiar readers to gain background necessary to replicate the paper}{(yes/no)}
yes

\end{itemize}
\checksubsection{Theoretical Contributions}
\begin{itemize}

\question{Does this paper make theoretical contributions?}{(yes/no)}
yes

	\ifyespoints{\vspace{1.2em}If yes, please address the following points:}
        \begin{itemize}
	
	\question{All assumptions and restrictions are stated clearly and formally}{(yes/partial/no)}
	yes

	\question{All novel claims are stated formally (e.g., in theorem statements)}{(yes/partial/no)}
	yes

	\question{Proofs of all novel claims are included}{(yes/partial/no)}
	yes

	\question{Proof sketches or intuitions are given for complex and/or novel results}{(yes/partial/no)}
	yes

	\question{Appropriate citations to theoretical tools used are given}{(yes/partial/no)}
	yes

	\question{All theoretical claims are demonstrated empirically to hold}{(yes/partial/no/NA)}
	partial

	\question{All experimental code used to eliminate or disprove claims is included}{(yes/no/NA)}
	no
	
	\end{itemize}
\end{itemize}

\checksubsection{Dataset Usage}
\begin{itemize}

\question{Does this paper rely on one or more datasets?}{(yes/no)}
yes

\ifyespoints{If yes, please address the following points:}
\begin{itemize}

	\question{A motivation is given for why the experiments are conducted on the selected datasets}{(yes/partial/no/NA)}
	partial

	\question{All novel datasets introduced in this paper are included in a data appendix}{(yes/partial/no/NA)}
	NA

	\question{All novel datasets introduced in this paper will be made publicly available upon publication of the paper with a license that allows free usage for research purposes}{(yes/partial/no/NA)}
	NA

	\question{All datasets drawn from the existing literature (potentially including authors' own previously published work) are accompanied by appropriate citations}{(yes/no/NA)}
	no

	\question{All datasets drawn from the existing literature (potentially including authors' own previously published work) are publicly available}{(yes/partial/no/NA)}
	yes

	\question{All datasets that are not publicly available are described in detail, with explanation why publicly available alternatives are not scientifically satisficing}{(yes/partial/no/NA)}
	NA

\end{itemize}
\end{itemize}

\checksubsection{Computational Experiments}
\begin{itemize}

\question{Does this paper include computational experiments?}{(yes/no)}
yes

\ifyespoints{If yes, please address the following points:}
\begin{itemize}

	\question{This paper states the number and range of values tried per (hyper-) parameter during development of the paper, along with the criterion used for selecting the final parameter setting}{(yes/partial/no/NA)}
	no

	\question{Any code required for pre-processing data is included in the appendix}{(yes/partial/no)}
	no

	\question{All source code required for conducting and analyzing the experiments is included in a code appendix}{(yes/partial/no)}
	no

	\question{All source code required for conducting and analyzing the experiments will be made publicly available upon publication of the paper with a license that allows free usage for research purposes}{(yes/partial/no)}
	no
        
	\question{All source code implementing new methods have comments detailing the implementation, with references to the paper where each step comes from}{(yes/partial/no)}
	no

	\question{If an algorithm depends on randomness, then the method used for setting seeds is described in a way sufficient to allow replication of results}{(yes/partial/no/NA)}
	partial

	\question{This paper specifies the computing infrastructure used for running experiments (hardware and software), including GPU/CPU models; amount of memory; operating system; names and versions of relevant software libraries and frameworks}{(yes/partial/no)}
	no

	\question{This paper formally describes evaluation metrics used and explains the motivation for choosing these metrics}{(yes/partial/no)}
	yes

	\question{This paper states the number of algorithm runs used to compute each reported result}{(yes/no)}
	yes

	\question{Analysis of experiments goes beyond single-dimensional summaries of performance (e.g., average; median) to include measures of variation, confidence, or other distributional information}{(yes/no)}
	yes

	\question{The significance of any improvement or decrease in performance is judged using appropriate statistical tests (e.g., Wilcoxon signed-rank)}{(yes/partial/no)}
	partial

	\question{This paper lists all final (hyper-)parameters used for each model/algorithm in the paper’s experiments}{(yes/partial/no/NA)}
	partial

\end{itemize}
\end{itemize}
\ifreproStandalone
\end{document}
\fi
\endgroup

\newcommand{\isAppendixMainFile}{}
\makeatletter
\@ifundefined{isAppendixMainFile}{
  \newif\ifappendixStandalone
  \appendixStandalonetrue
}{
  \newif\ifappendixStandalone
  \appendixStandalonefalse
}
\makeatother

\ifappendixStandalone
  \def\finishappendixdocument{\end{document}}
\else
  \let\finishappendixdocument\relax
\fi

\ifappendixStandalone
\documentclass[letterpaper]{article}

\usepackage{aaai2027}

\usepackage{amsmath}
\usepackage{amssymb}
\usepackage{amsthm}

\newtheorem{theorem}{Theorem}
\newtheorem{corollary}[theorem]{Corollary}
\newtheorem{proposition}[theorem]{Proposition}

\usepackage{array}
\usepackage{booktabs}
\usepackage{tabularx}

\usepackage{graphicx}
\usepackage{subcaption}
\usepackage{float}
\usepackage{placeins}

\usepackage{algorithm}
\usepackage{algpseudocode}

\usepackage{natbib}

\usepackage{microtype}
\usepackage{enumitem}

\usepackage[most]{tcolorbox}
\fi

\newcolumntype{Y}{>{\raggedright\arraybackslash}X}

\setcounter{topnumber}{3}
\setcounter{bottomnumber}{2}
\setcounter{totalnumber}{5}
\renewcommand{\topfraction}{0.90}
\renewcommand{\bottomfraction}{0.80}
\renewcommand{\textfraction}{0.08}
\renewcommand{\floatpagefraction}{0.82}
\renewcommand{\dbltopfraction}{0.90}
\renewcommand{\dblfloatpagefraction}{0.82}

\newcommand{\resultfigure}[2]{%
  \IfFileExists{#1}{%
    \includegraphics[pagebox=cropbox,width=#2]{#1}%
  }{%
    \fbox{\parbox[c][1.28in][c]{0.86\linewidth}{\centering\small
      Result figure placeholder\\[3pt]\texttt{\detokenize{#1}}}}%
  }%
}

\makeatletter
\setlength{\@fptop}{0pt}
\setlength{\@dblfptop}{0pt}

\ifappendixStandalone
\let\aaai@original@maketitle\@maketitle
\def\@maketitle{%
  \aaai@original@maketitle
  \begin{center}
    \refstepcounter{figure}%
    \includegraphics[width=\textwidth]{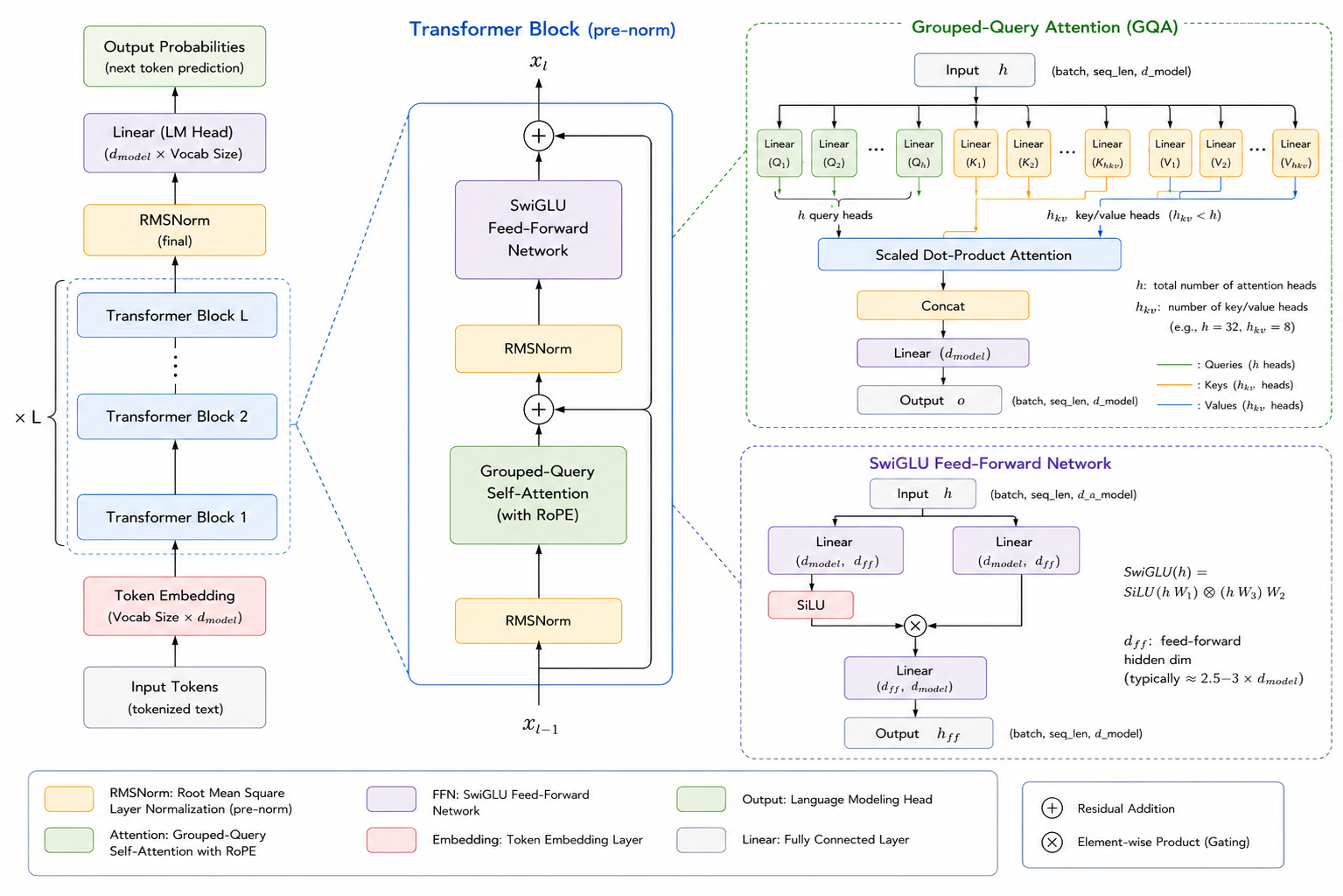}\par
    \vspace{2pt}
    \begin{minipage}{1.0\textwidth}
      \small
      \textbf{Figure~\thefigure:} Decoder-only Transformer architecture
      shared by all model configurations. Each pre-normalized block contains
      grouped-query self-attention with rotary positional embeddings, followed
      by a SwiGLU feed-forward network; residual connections surround both
      sublayers.
      \label{fig:model-architecture}
    \end{minipage}
  \end{center}
  \vspace{0.10in}
}
\fi
\makeatother

\ifappendixStandalone
\title{Supplementary Material}
\author{}
\affiliations{}

\begin{document}
\maketitle
\else
\twocolumn[{%
\begin{@twocolumnfalse}
  \section*{Supplementary Material}
  \begin{center}
    \refstepcounter{figure}%
    \includegraphics[width=0.95\textwidth]{architecture.png}\par
    \vspace{2pt}
    \begin{minipage}{0.94\textwidth}
      \small
      \textbf{Figure~\thefigure:} Decoder-only Transformer architecture
      shared by all model configurations. Each pre-normalized block contains
      grouped-query self-attention with rotary positional embeddings, followed
      by a SwiGLU feed-forward network; residual connections surround both
      sublayers.
      \label{fig:model-architecture}
    \end{minipage}
  \end{center}
\end{@twocolumnfalse}
}]
\fi

\appendix
\raggedbottom

\section{Model Architecture and Pretraining Details}
\label{app:model-training}

\subsection{Model Architecture}

All models use the same Llama-style decoder-only Transformer architecture. Each Transformer block adopts a pre-normalization structure consisting of RMSNorm, grouped-query self-attention with rotary positional embeddings (RoPE), a SwiGLU feed-forward network, and residual connections around both sublayers. A final RMSNorm and a linear language-modeling head map the hidden representations to next-token logits. We use the same tokenizer and vocabulary for all model configurations and do not apply supervised fine-tuning, RLHF, or any other post-training procedure.

Figure~\ref{fig:model-architecture} illustrates the shared model architecture and the internal structures of the grouped-query attention and SwiGLU modules.

\subsection{Model Configurations}

We construct six model configurations by varying the residual-stream width $N$ from $512$ to $3072$ while fixing the number of Transformer blocks to $L=32$. The attention head dimension is fixed at $128$. Accordingly, the number of query heads increases proportionally with $N$. We use a $4{:}1$ ratio between query heads and key--value heads, and set the SwiGLU intermediate dimension to $d_{\mathrm{ffn}}=3.5N$.

Table~\ref{tab:model-configurations} summarizes the resulting architectural
configurations and total parameter counts.

\begin{table}[H]
\centering
\scriptsize
\setlength{\tabcolsep}{3pt}
\renewcommand{\arraystretch}{1.08}
\resizebox{\columnwidth}{!}{%
\begin{tabular}{@{}lrrrrrr@{}}
\toprule
Model & $L$ & $N$ & Q heads & KV heads & $d_{\mathrm{ffn}}$ & Params. \\
\midrule
GPT-W512  & 32 & 512  & 4  & 1 & 1792  & 0.1524B \\
GPT-W1024 & 32 & 1024 & 8  & 2 & 3584  & 0.5229B \\
GPT-W1536 & 32 & 1536 & 12 & 3 & 5376  & 1.1114B \\
GPT-W2048 & 32 & 2048 & 16 & 4 & 7168  & 1.7450B \\
GPT-W2560 & 32 & 2560 & 20 & 5 & 8960  & 2.7265B \\
GPT-W3072 & 32 & 3072 & 24 & 6 & 10752 & 3.9261B \\
\bottomrule
\end{tabular}
}
\caption{Architectural configurations used in the width-scaling experiments.
All models contain 32 Transformer blocks, use an attention-head dimension of
128, and maintain a $4{:}1$ query-to-key/value head ratio.}
\label{tab:model-configurations}
\end{table}

The depth, tokenizer, vocabulary, context length, training-token budget, data mixture, and optimization protocol are held fixed across all configurations. The residual-stream width, number of attention heads, number of key--value heads, feed-forward dimension, and total parameter count therefore scale jointly with $N$. This construction isolates width-dependent architectural scaling while controlling the training protocol and data budget.

\subsection{Pretraining Data}

All models are pretrained on the same multi-source text mixture. The mixture contains general web text, educational web documents, deduplicated multi-domain text, WebText-style documents, encyclopedic content, technical question--answer data, and scientific articles. The approximate sampling proportions are provided in Table~\ref{tab:data-mixture}.
\begin{table}[H]
\centering
\small
\setlength{\tabcolsep}{8pt}
\renewcommand{\arraystretch}{1.10}
\begin{tabular}{@{}lr@{}}
\toprule
Data source & Proportion \\
\midrule
C4 & $35\%$ \\
FineWeb-Edu & $30\%$ \\
SlimPajama & $20\%$ \\
OpenWebText2 & $10\%$ \\
Wikipedia, StackExchange, and arXiv & $5\%$ \\
\midrule
Total & $100\%$ \\
\bottomrule
\end{tabular}
\caption{Pretraining data mixture shared by all model configurations.
Percentages denote approximate token-level sampling proportions.}
\label{tab:data-mixture}
\end{table}

\subsection{Pretraining Protocol}

All model configurations are pretrained from scratch under the same optimization protocol. We use AdamW with $(\beta_1,\beta_2)=(0.9,0.95)$, numerical stability constant $\epsilon=10^{-8}$, and weight decay $0.1$. Weight decay is excluded from embedding parameters, RMSNorm parameters, and bias terms. The learning rate is linearly warmed up for $2{,}000$ optimization steps and subsequently decayed according to a cosine schedule. Gradients are clipped to a maximum norm of $1.0$, and training is performed in BF16 precision without dropout.

All models use a context length of $2048$ tokens and are trained on a budget of $10$ billion tokens. The $10$B-token runs serve as the primary controlled experiments in this work; larger-token training budgets are left for subsequent scaling experiments.

\begin{table}[!t]
\centering
\scriptsize
\setlength{\tabcolsep}{4pt}
\renewcommand{\arraystretch}{0.98}
\begin{tabularx}{\columnwidth}{@{}>{\raggedright\arraybackslash}p{0.41\columnwidth}Y@{}}
\toprule
Hyperparameter & Setting \\
\midrule
\multicolumn{2}{@{}l}{\textit{Architecture}} \\
Architecture & Decoder-only Transformer \\
Normalization & RMSNorm, pre-normalization \\
Position encoding & RoPE \\
Attention & Grouped-query attention \\
Feed-forward network & SwiGLU \\
Tokenizer & Llama~3 tokenizer \\
Initialization & Normal, $\sigma=0.02$ \\
Optimizer & AdamW \\
Adam $(\beta_1,\beta_2)$ & $(0.9,0.95)$ \\
Adam $\epsilon$ & $10^{-8}$ \\
Weight decay & $0.1$ \\
No weight decay & Embeddings, RMSNorm, biases \\
Learning-rate schedule & Linear warmup + cosine decay \\
Warmup & $2{,}000$ steps \\
Peak learning rate & $3\times10^{-4}$ \\
Minimum learning rate & $1.5\times10^{-7}$ \\
Gradient clipping & $1.0$ \\
Precision & BF16 \\
Dropout & $0$ \\
Global batch size & $480$ sequences \\
Sequence length & $2048$ tokens \\
Training budget & $10$B tokens \\
Independent runs & $8$ per model configuration \\
Training seeds & $1234$, $520$, $314$, $315$, $3$, $1$, $5$, $31$ \\
\bottomrule
\end{tabularx}
\caption{Shared pretraining hyperparameters.}
\label{tab:pretraining-hyperparameters}
\end{table}

\subsection{Independent Training Runs}

For each model configuration, we perform eight independent pretraining runs using random seeds. Each run starts from an independently initialized parameter state. When controlled by the global random seed, the training-data order and other stochastic training operations are also independently randomized. This produces $6\times8=48$ independently trained models in total.

Unless otherwise stated, reported results are averaged over the eight independent training runs. Per-run results are retained to quantify the variability induced by parameter initialization and stochastic optimization.

\section{Training Dynamics and Optimization Stability}
\label{app:training-dynamics}

\raggedbottom

\subsection{Training Loss}

Training loss is computed from the token-level cross-entropy of each optimization batch. Figure~\ref{fig:training-loss-all-widths} shows the complete training-loss trajectories for all six model widths under the shared token budget and optimization schedule. Figure~\ref{fig:training-loss-small-widths} provides a single-column view of the three smaller configurations, making their relative convergence behavior easier to distinguish.

\begingroup
\setlength{\intextsep}{5pt}
\setlength{\abovecaptionskip}{3pt}
\setlength{\belowcaptionskip}{0pt}
\begin{figure}[H]
  \centering
  \resultfigure{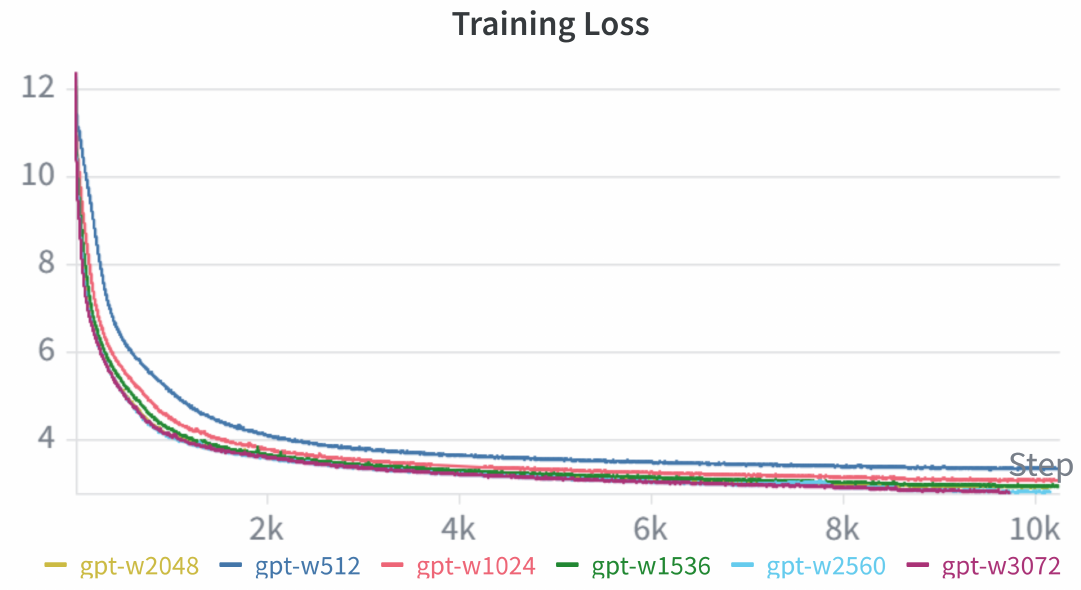}{\columnwidth}
  \caption{Complete training-loss trajectories across optimization steps for
  all six model widths.}
  \label{fig:training-loss-all-widths}
\end{figure}

\begin{figure}[H]
  \centering
  \resultfigure{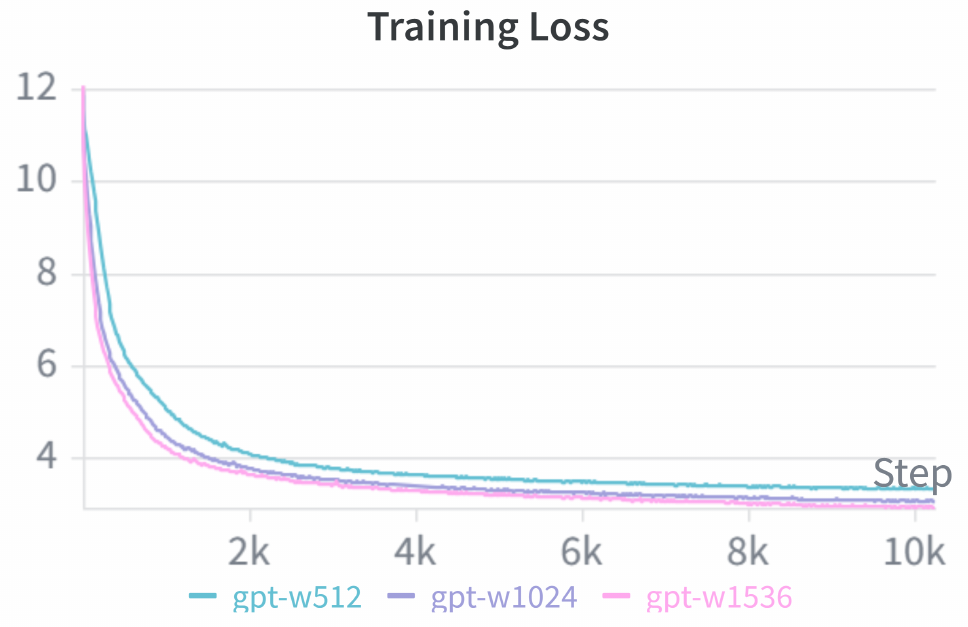}{\columnwidth}
  \caption{Training loss across optimization steps for the three smaller model
  widths: GPT-W512, GPT-W1024, and GPT-W1536.}
  \label{fig:training-loss-small-widths}
\end{figure}
\FloatBarrier
\endgroup

Across widths, the curves provide a direct view of optimization progress under the controlled training protocol. Short-timescale fluctuations are expected from minibatch sampling and do not by themselves indicate instability; the relevant diagnostic is the longer-term convergence trend.

\subsection{Validation Loss and Late-Training Perplexity}

Validation loss is evaluated on the fixed held-out split at regular intervals using the same tokenization, causal masking, and token-level normalization as in training. Figure~\ref{fig:validation-loss-small-widths} shows the validation loss trajectories for the three smaller model configurations.

To summarize validation behavior near the end of training, we compute perplexity as
\begin{equation}
  \operatorname{PPL}_{\mathrm{val}}
  = \exp\!\left(\mathcal{L}_{\mathrm{val}}\right),
\end{equation}
where $\mathcal{L}_{\mathrm{val}}$ is the mean token-level negative log-likelihood. Figure~\ref{fig:late-training-validation-ppl} reports validation perplexity over the late-training checkpoints for all six widths. Because perplexity is an exponential transformation of validation loss, small loss differences can appear more pronounced on the perplexity scale; all comparisons therefore use an identical evaluation corpus and loss normalization.

\begin{figure}[!t]
  \centering
  \resultfigure{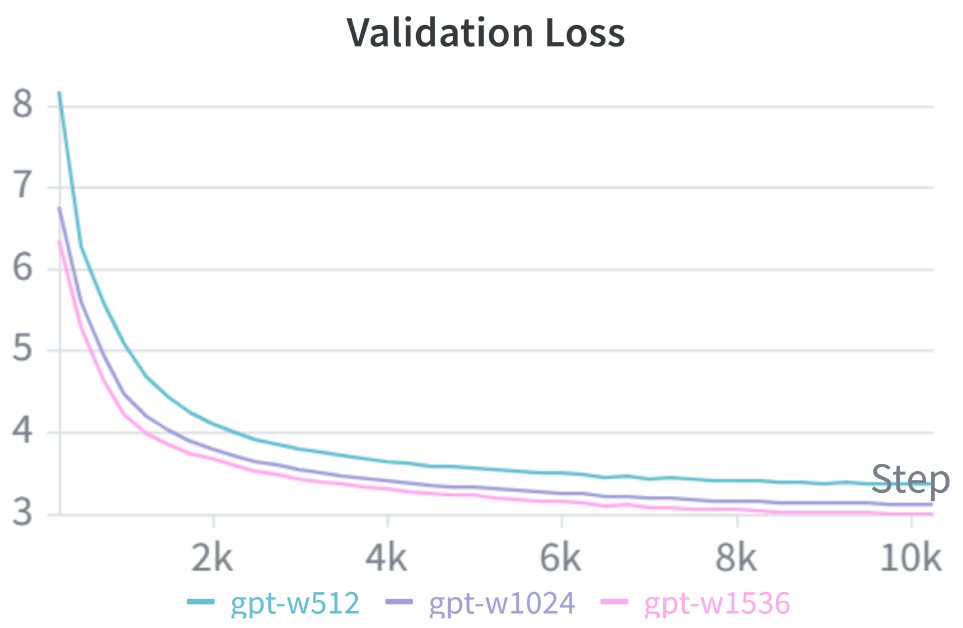}{\columnwidth}
  \caption{Validation loss across optimization steps for GPT-W512, GPT-W1024,
  and GPT-W1536.}
  \label{fig:validation-loss-small-widths}
\end{figure}

\begin{figure}[!t]
  \centering
  \resultfigure{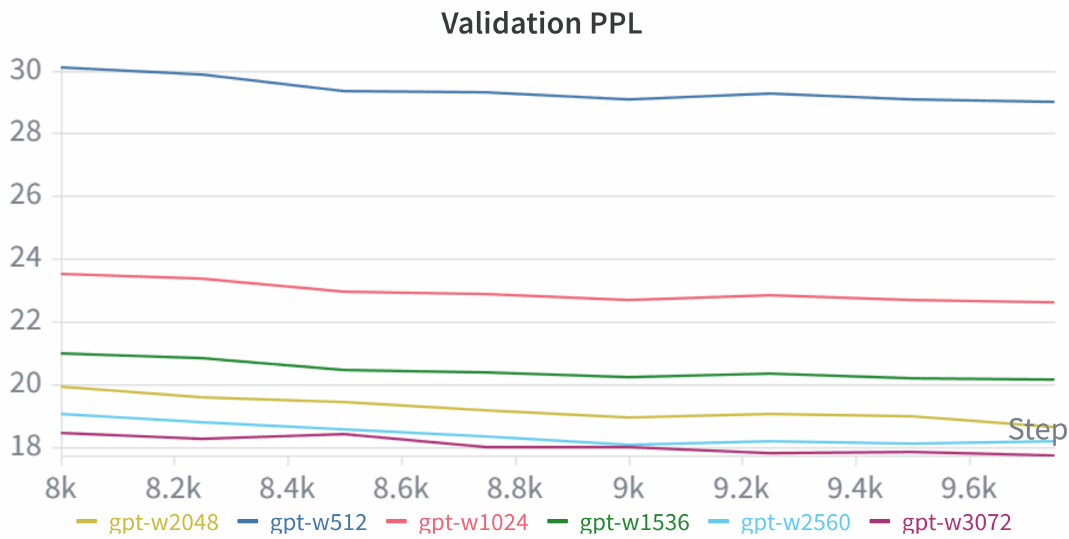}{\columnwidth}
  \caption{Validation perplexity during the final portion of training for all six model widths.}
  \label{fig:late-training-validation-ppl}
\end{figure}

\begin{figure*}[!t]
  \centering
  \resultfigure{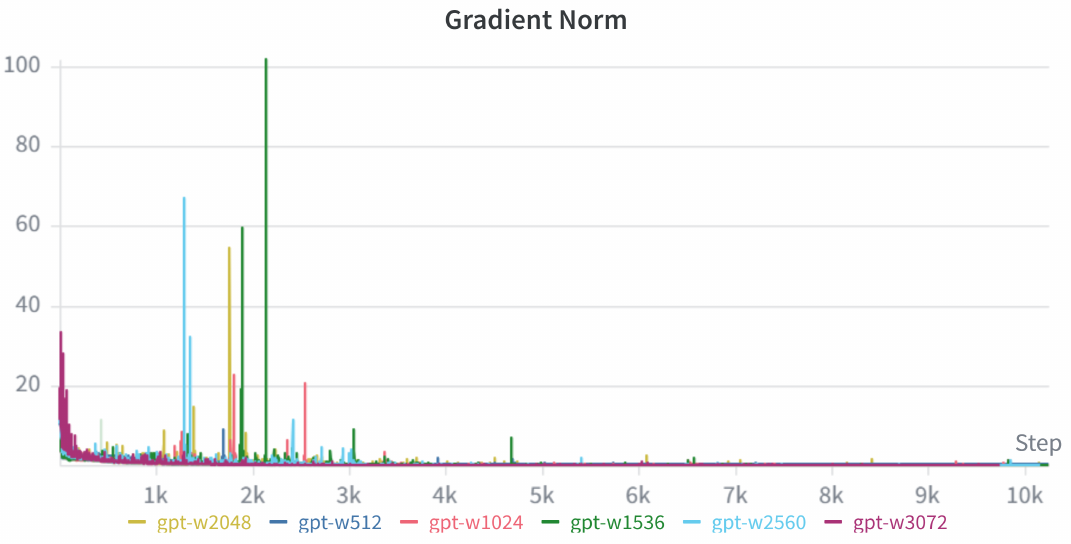}{\textwidth}
  \caption{Global gradient $\ell_2$ norm before clipping across optimization steps. Curves show means over eight independent runs, shaded regions show one standard deviation across seeds, and the horizontal dashed line marks the clipping threshold of $1.0$.}
  \label{fig:gradient-norm}
\end{figure*}

\subsection{Gradient Norm and Training Stability}

We track the global $\ell_2$ norm of the gradients before clipping throughout pretraining. For parameters $\theta$ and step-$t$ gradient $g_t$, the reported
quantity is
\begin{equation}
  \lVert g_t\rVert_2
  = \left(\sum_{p\in\theta}\lVert g_{t,p}\rVert_2^2\right)^{1/2}.
\end{equation}
The norm is measured before applying the maximum-norm threshold of $1.0$ so that the diagnostic retains information about gradient spikes. Values above the threshold indicate steps on which clipping rescales the update; they do not indicate that the post-clipping update norm exceeds the threshold.

Figure~\ref{fig:gradient-norm} shows the gradient-norm trajectories for all model widths. We assess stability using the overall scale and trend of the norms, the frequency and duration of excursions above the clipping threshold, and agreement across seeds. Isolated spikes can arise from unusually difficult minibatches, while persistent growth, repeated extreme excursions, non-finite values, or a simultaneous increase in validation loss would be evidence of optimization instability. Reporting the pre-clipping norm together with the loss curves makes it possible to distinguish benign clipped spikes from instability that materially affects convergence.

\subsection{Pythia Model Configurations}
\label{app:pythia-configurations}

We evaluate five models from the Pythia family. The small-scale regime contains Pythia-14M, Pythia-31M, and Pythia-70M, all of which use six Transformer layers and hidden widths \(N\in\{128,256,512\}\). The large-scale regime contains Pythia-2.8B and Pythia-6.9B, both of which use 32 Transformer layers and hidden widths \(N\in\{2560,4096\}\). This grouping supports controlled within-depth comparisons of activation-gradient geometry as hidden width increases. The architectural configurations are summarized in Table~\ref{tab:pythia-configurations}.

\begin{table}[H]
\centering
\scriptsize
\setlength{\tabcolsep}{2pt}
\renewcommand{\arraystretch}{1.15}
\resizebox{\columnwidth}{!}{%
\begin{tabular}{lcccccccc}
\toprule
Model
& Layers
& Width \(N\)
& Heads \(H\)
& \(d_{\mathrm{head}}\)
& \(d_{\mathrm{ff}}\)
& Context
& Vocabulary
& Residual type
\\
\midrule
Pythia-14M
& 6
& 128
& 4
& 32
& 512
& 2048
& 50304
& Parallel
\\
Pythia-31M
& 6
& 256
& 8
& 32
& 1024
& 2048
& 50304
& Parallel
\\
Pythia-70M
& 6
& 512
& 8
& 64
& 2048
& 2048
& 50304
& Parallel
\\
Pythia-2.8B
& 32
& 2560
& 32
& 80
& 10240
& 2048
& 50304
& Parallel
\\
Pythia-6.9B
& 32
& 4096
& 32
& 128
& 16384
& 2048
& 50304
& Parallel
\\
\bottomrule
\end{tabular}
}
\caption{Architectural configurations of the Pythia models used in our experiments. \(N\) denotes the hidden width, \(H\) denotes the number of attention heads, \(d_{\mathrm{head}}=N/H\) denotes the dimension of each attention head, and \(d_{\mathrm{ff}}\) denotes the feed-forward intermediate dimension.}
\label{tab:pythia-configurations}
\end{table}

All evaluated models use the GPT-NeoX decoder-only architecture with GELU activations, parallel attention and feed-forward residual branches, untied input and output embeddings, and rotary positional embeddings. Rotary embeddings are applied to \(25\%\) of each attention head using a rotary base of \(10{,}000\). As shown in Table~\ref{tab:pythia-configurations}, the feed-forward dimension satisfies \(d_{\mathrm{ff}}=4N\) for all evaluated configurations, whereas the number and dimension of attention heads follow the original Pythia architectures.

For Pythia-14M, Pythia-31M, and Pythia-70M, we evaluate ten independently trained runs for each configuration. These runs share the same architecture and training protocol but differ in their training random seeds. For Pythia-2.8B and Pythia-6.9B, one pretrained checkpoint is evaluated at each width, and variability in the activation-gradient measurements is estimated across independently sampled evaluation subsets.

We additionally evaluate the conventional language-modeling performance of all five Pythia configurations on WikiText-2. All perplexity values are recomputed using the same preprocessing and evaluation protocol. The results are reported in Table~\ref{tab:pythia-recomputed-ppl}.

\begin{table}[H]
\centering
\small
\setlength{\tabcolsep}{12pt}
\renewcommand{\arraystretch}{1.15}
\begin{tabular}{lc}
\toprule
Model & WikiText-2 PPL \(\downarrow\) \\
\midrule
Pythia-14M  & 169.26 \\
Pythia-31M  & 114.24 \\
Pythia-70M  & 93.65  \\
Pythia-2.8B & 17.71  \\
Pythia-6.9B & 15.85  \\
\bottomrule
\end{tabular}
\caption{Recomputed WikiText-2 perplexity of the evaluated Pythia model configurations. Lower values indicate better language-modeling performance.}
\label{tab:pythia-recomputed-ppl}
\end{table}

As shown in Table~\ref{tab:pythia-recomputed-ppl}, WikiText-2 perplexity decreases consistently across the evaluated Pythia configurations as model width and overall capacity increase. These conventional language-modeling results verify that the evaluated models exhibit the expected predictive-performance ordering and complement the activation-gradient geometry analyzed in the main experiments.

\subsection{ResNet-20 Experimental Configuration and Accuracy}
\label{app:resnet20-configuration}

We evaluate width-dependent behavior on CIFAR-10 and CIFAR-100 using a standard ResNet-20 architecture composed of BasicBlocks. Each input image has resolution \(3\times32\times32\), and the network contains three residual stages with three BasicBlocks per stage. For width multiplier \(w\in\{0.5,1,2,4\}\), the channel dimensions of the three stages are \([16w,32w,64w]\). The first blocks of the second and third stages perform spatial downsampling using stride \(2\). Following the original CIFAR ResNet design, we use the parameter-free Option-A shortcut, which applies strided subsampling and zero padding when the spatial resolution and number of channels change.

The complete architectural and optimization settings are summarized in Table~\ref{tab:resnet20-training-config}. Models are trained using standard cross-entropy loss without label smoothing. All convolutional layers are initialized using Kaiming-normal initialization. For every dataset--width configuration, we independently train eight models using random seeds \(0\) through \(7\).

\begin{table}[H]
\centering
\small
\setlength{\tabcolsep}{5pt}
\renewcommand{\arraystretch}{1.12}
\begin{tabular}{ll}
\toprule
Setting & Value \\
\midrule
Datasets & CIFAR-10, CIFAR-100 \\
Input resolution & \(3\times32\times32\) \\
Architecture & ResNet-20 with BasicBlocks \\
Blocks per stage & \(3,3,3\) \\
Width multipliers & \(0.5\times,1\times,2\times,4\times\) \\
Stage channels & \([16w,32w,64w]\) \\
Downsampling & Stage 2 and 3, stride \(2\) \\
Shortcut & Option A with zero padding \\
Loss & Cross-entropy \\
Label smoothing & None \\
Initialization & Kaiming normal \\
Optimizer & SGD \\
Initial learning rate & \(0.1\) \\
Momentum & \(0.9\) \\
Weight decay & \(10^{-4}\) \\
Batch size & \(128\) \\
Training steps & \(64{,}000\) \\
LR scheduler & MultiStepLR \\
LR milestones & \(32{,}000,\ 48{,}000\) steps \\
LR decay factor & \(0.1\) \\
Learning-rate schedule & \(0.1\rightarrow0.01\rightarrow0.001\) \\
Evaluation interval & Every \(8{,}000\) steps \\
Logging interval & Every \(500\) steps \\
Training seeds & \(0,1,2,3,4,5,6,7\) \\
\bottomrule
\end{tabular}
\caption{Architecture and training configuration used for the ResNet-20 experiments on CIFAR-10 and CIFAR-100.}
\label{tab:resnet20-training-config}
\end{table}

Table~\ref{tab:resnet20-accuracy} reports the final test accuracy obtained after \(64{,}000\) training steps. For each dataset--width configuration, the final accuracy is summarized as the mean and standard deviation across eight independently trained models. The parameter count includes all trainable parameters, including the dataset-dependent classification head.

\begin{table}[H]
\centering
\small
\setlength{\tabcolsep}{4pt}
\renewcommand{\arraystretch}{1.12}
\resizebox{\columnwidth}{!}{%
\begin{tabular}{lccc}
\toprule
Dataset
& Width
& Parameters
& Final test accuracy (\%)
\\
\midrule
CIFAR-10
& \(0.5\times\)
& \(68{,}050\)
& \(88.21 \pm 0.21\)
\\
CIFAR-10
& \(1\times\)
& \(269{,}722\)
& \(91.56 \pm 0.30\)
\\
CIFAR-10
& \(2\times\)
& \(1{,}073{,}962\)
& \(93.61 \pm 0.15\)
\\
CIFAR-10
& \(4\times\)
& \(4{,}286{,}026\)
& \(94.49 \pm 0.35\)
\\
\midrule
CIFAR-100
& \(0.5\times\)
& \(71{,}020\)
& \(58.25 \pm 0.57\)
\\
CIFAR-100
& \(1\times\)
& \(275{,}572\)
& \(67.19 \pm 0.25\)
\\
CIFAR-100
& \(2\times\)
& \(1{,}085{,}572\)
& \(72.00 \pm 0.27\)
\\
CIFAR-100
& \(4\times\)
& \(4{,}309{,}156\)
& \(75.67 \pm 0.27\)
\\
\bottomrule
\end{tabular}
}
\caption{Final test accuracy of the evaluated ResNet-20 width configurations. Results are reported as mean \(\pm\) standard deviation across eight independently trained models.}
\label{tab:resnet20-accuracy}
\end{table}

As shown in Table~\ref{tab:resnet20-accuracy}, increasing the channel-width multiplier consistently improves mean final test accuracy on both CIFAR-10 and CIFAR-100, confirming the expected width-dependent predictive-performance improvement across the evaluated checkpoints.

\section{Computation of Activation-Gradient Alignment Statistics}
\label{app:alignment-estimation}

For each evaluation unit \((x_i,y_i)\), we extract the activation
gradient
\[
q_i
=
\nabla_{z_i}
\ell\!\left(f_{\mathrm{top}}(z_i),y_i\right),
\qquad
z_i=f_{\mathrm{bot}}(x_i).
\]
Given \(n\) activation-gradient vectors, let
\[
\widehat{\mu}
=
\frac{1}{n}\sum_{i=1}^{n}q_i,
\qquad
E_{i,:}
=
(q_i-\widehat{\mu})^\top.
\]
The empirical quantities are computed as
\[
\widehat v_{\parallel}
=
\frac{\|E\widehat{\mu}\|_2^2}{n-1},
\qquad
\widehat v_2
=
\frac{\|EE^\top\|_F^2}{(n-1)^2}.
\]
Thus,
\[
\widehat d_{\parallel}
=
\frac{\|\widehat{\mu}\|_2^4}
{\widehat v_{\parallel}},
\qquad
\widehat d_2
=
\frac{\|\widehat{\mu}\|_2^4}
{\widehat v_2},
\qquad
\widehat d_{\mathrm{align}}
=
\min\{\widehat d_{\parallel},\widehat d_2\}.
\]

\begin{algorithm}[H]
\footnotesize
\begin{algorithmic}[1]
\Require Model \(f=f_{\mathrm{top}}\circ f_{\mathrm{bot}}\), evaluation samples \(\{(x_i,y_i)\}_{i=1}^{n}\)
\Ensure \(\widehat r_q,\widehat d_{\parallel},\widehat d_2,\widehat d_{\mathrm{align}}\)

\For{\(i=1,\ldots,n\)}
    \State \(z_i\gets f_{\mathrm{bot}}(x_i)\)
    \State \(q_i\gets\nabla_{z_i}\ell(f_{\mathrm{top}}(z_i),y_i)\)
\EndFor

\State \(\widehat r_q\gets n^{-1}\sum_{i=1}^{n}\|q_i\|_2\)
\State \(\widehat\mu\gets n^{-1}\sum_{i=1}^{n}q_i\)
\State \(E_{i,:}\gets(q_i-\widehat\mu)^\top\)
\State \(s\gets\|\widehat\mu\|_2^2\)
\State \(v_{\parallel}\gets\|E\widehat\mu\|_2^2/(n-1)\)
\State \(v_2\gets\|EE^\top\|_F^2/(n-1)^2\)
\State \(\widehat d_{\parallel}\gets s^2/v_{\parallel}\)
\State \(\widehat d_2\gets s^2/v_2\)
\State \(\widehat d_{\mathrm{align}}\gets
\min\{\widehat d_{\parallel},\widehat d_2\}\)
\State \Return \(\widehat r_q,\widehat d_{\parallel},
\widehat d_2,\widehat d_{\mathrm{align}}\)
\end{algorithmic}
\caption{Estimation of activation-gradient alignment statistics. The
Gram-matrix computation avoids explicitly forming the \(N\times N\)
covariance matrix.}
\label{alg:alignment-dimension-estimation}
\end{algorithm}

The identity
\[
\operatorname{tr}(\widehat\Sigma^2)
=
\frac{\|EE^\top\|_F^2}{(n-1)^2}
\]
allows the computation to use an \(n\times n\) Gram matrix rather than
an \(N\times N\) covariance matrix. In implementation, denominators
smaller than a numerical tolerance are treated as zero, and the
corresponding effective dimension is set to \(+\infty\).

\section{Monte Carlo Estimation of Misalignment Probability}
\label{app:misalignment-estimation}

For each Monte Carlo trial, we independently sample \(M\) training-side
and \(K\) test-side activation gradients and compute
\[
\mu_M
=
\frac{1}{M}\sum_{i=1}^{M}q_i,
\qquad
g_K
=
\frac{1}{K}\sum_{j=1}^{K}\widetilde q_j.
\]
The empirical misalignment probability is
\[
\widehat p_{M,K}
=
\frac{1}{R}
\sum_{r=1}^{R}
\mathbf 1
\left\{
(\mu_M^{(r)})^\top g_K^{(r)}\leq0
\right\}.
\]

\begin{algorithm}[H]
\footnotesize
\begin{algorithmic}[1]
\Require Fixed, disjoint gradient pools \(\mathcal Q_{\mathrm{tr}}\) and \(\mathcal Q_{\mathrm{te}}\), sample sizes \(M,K\), trials \(R\)
\Ensure \(\widehat p_{M,K}\) and \(\widehat{\mathrm{SE}}_{M,K}\)

\State \(C\gets0\)
\For{\(r=1,\ldots,R\)}
    \State Sample \(\{q_i\}_{i=1}^{M}\) with replacement from \(\mathcal Q_{\mathrm{tr}}\)
    \State Sample \(\{\widetilde q_j\}_{j=1}^{K}\) with replacement from \(\mathcal Q_{\mathrm{te}}\)
    \State \(\mu_M\gets M^{-1}\sum_{i=1}^{M}q_i\)
    \State \(g_K\gets K^{-1}\sum_{j=1}^{K}\widetilde q_j\)
    \State \(C\gets C+\mathbf 1\{\mu_M^\top g_K\leq0\}\)
\EndFor
\State \(\widehat p_{M,K}\gets C/R\)
\State \(\widehat{\mathrm{SE}}_{M,K}\gets
\sqrt{\widehat p_{M,K}(1-\widehat p_{M,K})/R}\)
\State \Return \(\widehat p_{M,K},\widehat{\mathrm{SE}}_{M,K}\)
\end{algorithmic}
\caption{Monte Carlo estimation of the activation-gradient
misalignment probability. Training-side and test-side gradients are
sampled independently with replacement in every trial. Conditional on
the two fixed pools, the Monte Carlo trials are independent.}
\label{alg:misalignment-probability-estimation}
\end{algorithm}

Sampling is performed with replacement from the two fixed, disjoint
gradient pools; repeated gradients are therefore allowed within a
trial. Conditional on these pools, the estimator targets the empirical
misalignment probability induced by their two empirical distributions.
We use \(R=500{,}000\) conditionally independent trials for each model configuration
and each pair \((M,K)\). The worst-case Monte Carlo standard error is
therefore
\[
\widehat{\mathrm{SE}}_{M,K}
\leq
\frac{1}{2\sqrt{R}}
\approx
7.1\times10^{-4}.
\]

The reported plug-in certificate estimates are computed after
Algorithm~\ref{alg:misalignment-probability-estimation} as
\[
\widehat\Xi_{M,K}
=
\frac{1}{M\widehat d_{\parallel}}
+
\frac{1}{K\widehat d_{\parallel}}
+
\frac{1}{MK\widehat d_2},
\]
\[
\widehat B_{M,K}^{\mathrm{exact}}
=
\frac{\widehat\Xi_{M,K}}
{1+\widehat\Xi_{M,K}},
\qquad
\widehat B_{M,K}^{\mathrm{align}}
=
\frac{1}
{1+n_{\mathrm{eff}}\widehat d_{\mathrm{align}}},
\]
where
\[
n_{\mathrm{eff}}
=
\frac{MK}{M+K+1}.
\]

For configurations with multiple training seeds, both procedures are
applied separately to each trained model. Estimates are first
aggregated across independently sampled evaluation subsets within each
model and are subsequently summarized across training seeds.

\section{Component-wise Decomposition of the Effective Alignment Dimension}
\label{app:alignment-dimension-decomposition}

Recall that
\[
d_{\mathrm{align}}
=
\min\{d_{\parallel},d_2\}.
\]
To examine which component controls the effective alignment dimension,
we report $d_{\parallel}$, $d_2$, and $d_{\mathrm{align}}$ for the
controlled LLaMA-style models. For Pythia, the available table contains
only aggregate component summaries; we therefore do not infer an
aggregate minimum or bottleneck frequency from those summaries.

\paragraph{Controlled LLaMA-style models.}

\begin{table}[t]
\centering
\scriptsize
\setlength{\tabcolsep}{3.2pt}
\begin{tabular}{c c c c c}
\toprule
Width $N$
& $d_{\parallel}$
& $d_2$
& $d_{\mathrm{align}}$
& $d_2$ bottleneck \\
\midrule
512
& $11.736 \pm 0.694$
& $3.944 \pm 0.221$
& $3.944 \pm 0.221$
& $100\%$ \\

1024
& $10.620 \pm 0.749$
& $4.418 \pm 0.319$
& $4.418 \pm 0.319$
& $100\%$ \\

1536
& $11.347 \pm 0.748$
& $6.261 \pm 0.885$
& $6.261 \pm 0.885$
& $100\%$ \\

2048
& $9.729 \pm 0.746$
& $6.691 \pm 0.788$
& $6.691 \pm 0.788$
& $100\%$ \\

2560
& $9.056 \pm 0.432$
& $7.767 \pm 0.991$
& $7.635 \pm 0.883$
& $80\%$ \\

3072
& $11.740 \pm 0.748$
& $9.525 \pm 0.467$
& $9.523 \pm 0.463$
& $95\%$ \\
\bottomrule
\end{tabular}
\caption{Component-wise decomposition of the effective alignment dimension across the controlled LLaMA-style model widths. The final column reports the frequency with which $d_2$ is the active bottleneck.}
\label{tab:alignment-dimension-decomposition}
\end{table}

The covariance-energy component $d_2$ is the dominant bottleneck over most of the investigated width range. It determines $d_{\mathrm{align}}$ consistently up to $N=2048$, while $d_{\parallel}$ becomes limiting only occasionally at the two largest widths. Although $d_{\parallel}$ is not monotonic, $d_2$ and the resulting $d_{\mathrm{align}}$ exhibit a clear overall increase with width, with $d_{\mathrm{align}}$ increasing from $3.944$ at $N=512$ to $9.523$ at $N=3072$.

\paragraph{Pythia models.}

\begin{table}[t]
\centering
\scriptsize
\setlength{\tabcolsep}{8pt}
\begin{tabular}{c c c}
\toprule
Width $N$
& $d_{\parallel}$
& $d_2$ \\
\midrule
\multicolumn{3}{c}{\textit{6-layer Pythia models}} \\
\midrule
128
& 5.5915
& 3.4390 \\

256
& 7.6445
& 14.2963 \\

512
& 19.5501
& 14.3800 \\
\midrule
\multicolumn{3}{c}{\textit{32-layer Pythia models}} \\
\midrule
2560
& 7.9527
& 11.0866 \\

4096
& 9.3290
& 10.0440 \\
\bottomrule
\end{tabular}
\caption{Aggregate component summaries for the Pythia models.
Comparisons are made within each fixed-depth group. The table does not
derive an aggregate $d_{\mathrm{align}}$ or a bottleneck frequency from
the two component means.}
\label{tab:pythia-alignment-decomposition}
\end{table}

For Pythia, the valid aggregation order is to compute
\[
d_{\mathrm{align}}^{(r)}
=
\min\!\left\{
d_{\parallel}^{(r)},d_2^{(r)}
\right\}
\]
for each evaluation unit $r$ and only then aggregate the resulting
$d_{\mathrm{align}}^{(r)}$ values. Bottleneck frequencies must likewise
be computed from the same evaluation units. Neither quantity can in
general be recovered from the two aggregate component means in
Table~\ref{tab:pythia-alignment-decomposition}; accordingly, no
Pythia bottleneck claim is made from this table.

\section{Evolution of Gradient-Misalignment Probability during Pretraining}
\label{sec:misalignment-training-dynamics}

We next examine how finite-sample gradient misalignment evolves during
pretraining and whether the width-dependent alignment advantage persists
throughout the training trajectory. For a model of width $N$ evaluated after
$T$ pretraining tokens, we estimate
\begin{equation}
\widehat{p}_{M,K}(N;T)
=
\frac{1}{R}
\sum_{r=1}^{R}
\mathbf{1}
\left\{
\left(\mu_{M,N,T}^{(r)}\right)^{\top}
g_{K,N,T}^{(r)}
\leq 0
\right\},
\label{eq:empirical-checkpoint-misalignment}
\end{equation}
where $\mu_{M,N,T}^{(r)}$ and $g_{K,N,T}^{(r)}$ are independently
sampled training- and test-averaged gradient signals in the $r$-th trial.
Each sampled sequence contains $2048$ tokens, and the evaluation protocol is
held fixed across model widths and checkpoints within each comparison.

We first evaluate the width-$512$ and width-$1024$ LLaMA-style models at
checkpoints corresponding to $1$, $3$, $5$, $7$, and $9$ billion
pretraining tokens. We consider two finite-sample budgets,
$M=K=8$ and $M=K=16$, in order to determine whether the checkpoint-wise
behavior persists under different levels of gradient-estimation uncertainty.

\begin{figure*}[t]
    \centering
    \begin{subfigure}[t]{0.48\textwidth}
        \centering
        \includegraphics[width=\linewidth]{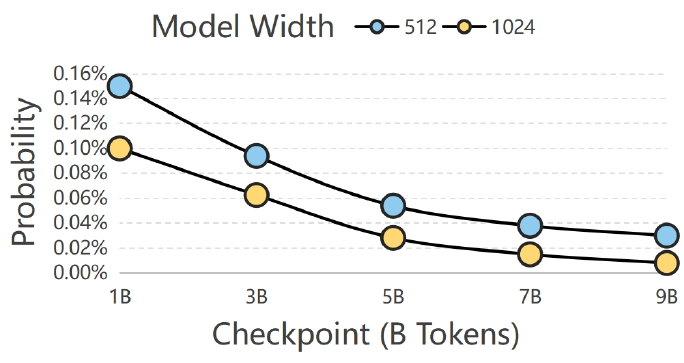}
        \caption{$M=K=8$.}
        \label{fig:misalignment-pretraining-8}
    \end{subfigure}
    \hfill
    \begin{subfigure}[t]{0.48\textwidth}
        \centering
        \includegraphics[width=\linewidth]{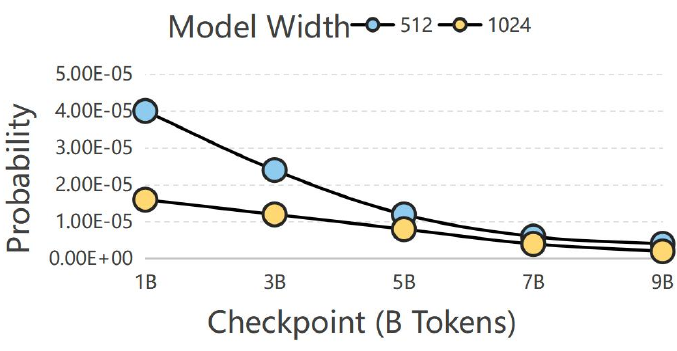}
        \caption{$M=K=16$.}
        \label{fig:misalignment-pretraining-16}
    \end{subfigure}
    \caption{
    Evolution of the empirical finite-sample gradient-misalignment
    probability during pretraining. Width-$512$ and width-$1024$
    LLaMA-style models are evaluated at checkpoints corresponding to
    $1$, $3$, $5$, $7$, and $9$ billion pretraining tokens. Each sampled
    sequence contains $2048$ tokens. Panel~(a) uses $M=K=8$, whereas
    panel~(b) uses $M=K=16$. For both sampling budgets, the empirical
    misalignment probability decreases with pretraining progress and remains
    lower for the wider model.
    }
    \label{fig:misalignment-training-dynamics}
\end{figure*}

\begin{figure*}[t]
    \centering
    \begin{subfigure}[t]{0.48\textwidth}
        \centering
        \includegraphics[width=\linewidth]
        {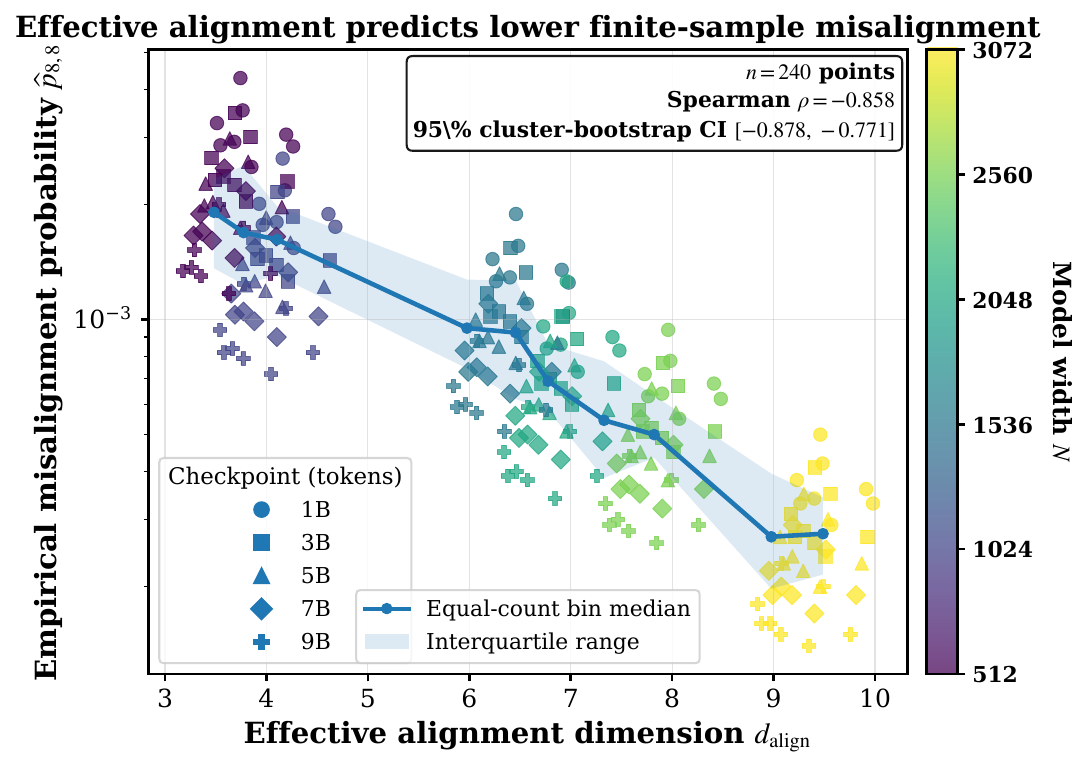}
        \caption{
        Empirical misalignment over the complete width--checkpoint grid.
        Each cell reports the mean over eight independently trained models.
        }
        \label{fig:full-misalignment-heatmap}
    \end{subfigure}
    \hfill
    \begin{subfigure}[t]{0.48\textwidth}
        \centering
        \includegraphics[width=\linewidth]
        {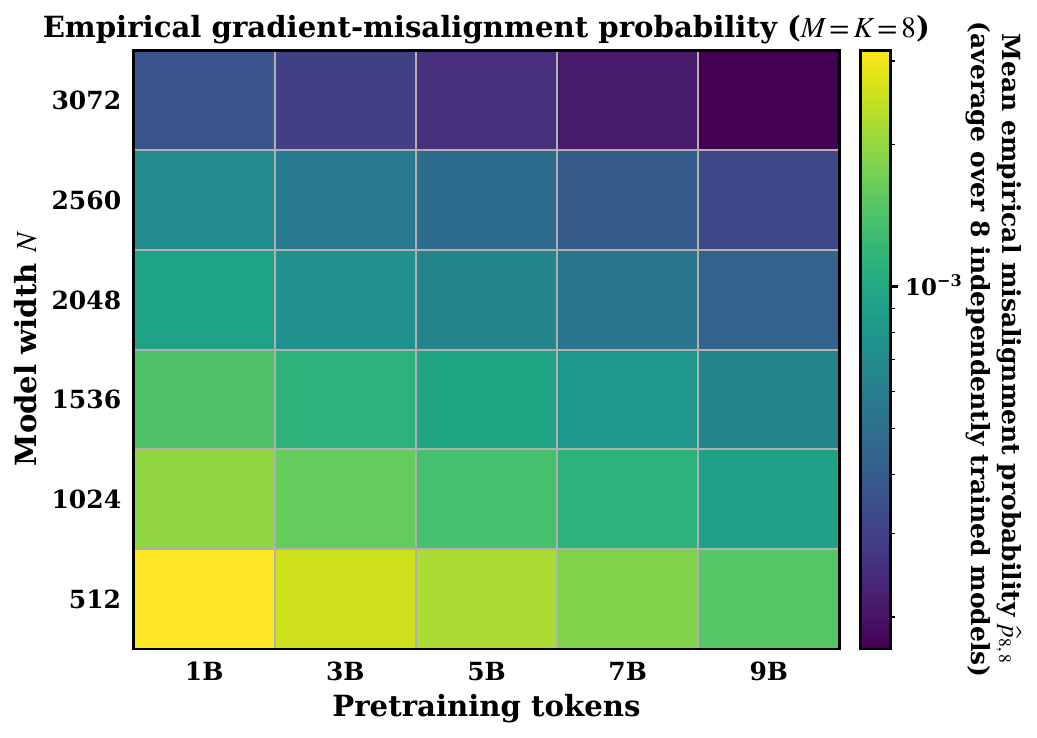}
        \caption{
        Relationship between $d_{\mathrm{align}}$ and
        $\widehat{p}_{8,8}$. Color denotes model width and marker shape
        denotes the pretraining checkpoint.
        }
        \label{fig:dalign-misalignment-scatter}
    \end{subfigure}
    \caption{
    Extended analysis of finite-sample gradient misalignment for $M=K=8$.
    Panel~(a) reports the complete $6\times5$ width--checkpoint grid.
    Each of the $30$ cells is computed by first estimating
    $\widehat{p}_{8,8}$ separately for each trained model and then averaging
    across eight training seeds. Panel~(b) retains all $240$
    width--checkpoint--seed observations. The solid curve reports medians
    over equal-count bins, and the shaded region denotes the corresponding
    interquartile range. Higher $d_{\mathrm{align}}$ is strongly associated
    with lower empirical misalignment (Spearman $\rho=-0.858$;
    trajectory-clustered bootstrap 95\% CI $[-0.878,-0.771]$).
    }
    \label{fig:extended-misalignment-analysis}
\end{figure*}

As shown in Figure~\ref{fig:misalignment-training-dynamics}, the empirical
gradient-misalignment probability decreases consistently as pretraining
progresses. For $M=K=8$, the probability for the width-$512$ model decreases
from approximately $0.15\%$ at the $1$B-token checkpoint to approximately
$0.03\%$ at the $9$B-token checkpoint. The width-$1024$ model follows the
same decreasing trajectory but remains below the narrower model at every
checkpoint, decreasing from approximately $0.10\%$ to below $0.01\%$.

The same qualitative behavior is observed for $M=K=16$, but at a
substantially smaller absolute probability scale. Increasing the sampling
budget reduces the uncertainty of both the training- and test-averaged
gradient estimates, while the width-$1024$ model continues to exhibit a
lower misalignment probability throughout pretraining. The width-dependent
advantage therefore remains visible even when the misalignment event becomes
rare.

Because the two widths are compared at matched pretraining-token checkpoints
and under identical sampling budgets within each panel, the observed
difference cannot be attributed to additional pretraining data or a larger
evaluation sample for the wider model. Instead, the results indicate that the
width-dependent alignment advantage emerges early and persists throughout the
investigated training trajectory.

\paragraph{Complete width--checkpoint analysis.}
We further test whether this behavior extends beyond the two representative
widths in Figure~\ref{fig:misalignment-training-dynamics}. For $M=K=8$, we
evaluate all six LLaMA-style widths
\begin{equation}
N\in
\{512,1024,1536,2048,2560,3072\}
\label{eq:misalignment-width-set}
\end{equation}
at the same five pretraining checkpoints. Each width is represented by eight
independently trained models, resulting in
\begin{equation}
6\ \text{widths}
\times
5\ \text{checkpoints}
\times
8\ \text{training seeds}
=
240
\label{eq:misalignment-observation-count}
\end{equation}
model--checkpoint observations.

For the heatmap, we first compute $\widehat{p}_{8,8}$ separately for each
trained model and then average the eight model-level estimates within each
width--checkpoint pair. The resulting grid therefore contains
$6\times5=30$ cells. The representative trajectories above and the complete
grid are estimated from independently sampled evaluation subsets, so their
absolute values need not coincide exactly; the relevant comparison is the
width- and checkpoint-wise ordering within each analysis.

Figure~\ref{fig:extended-misalignment-analysis}(a) shows a consistent
two-dimensional pattern. At every fixed checkpoint, increasing model width
reduces the empirical misalignment probability. At every fixed width,
additional pretraining also reduces the probability. In the complete-grid
evaluation, the mean probability for width $512$ decreases from
$3.16\times10^{-3}$ at $1$B tokens to $1.47\times10^{-3}$ at $9$B tokens,
whereas the corresponding probability for width $3072$ decreases from
$3.69\times10^{-4}$ to $1.71\times10^{-4}$.

Consequently, the width-$3072$ models exhibit an approximately $8.6$-fold
smaller misalignment probability than the width-$512$ models at both the
$1$B- and $9$B-token checkpoints. The ordering is not restricted to the two
extreme widths: each intermediate width produces a trajectory lying between
those of its narrower and wider counterparts.

The complete grid also shows that width and pretraining progress play
complementary roles. Additional pretraining shifts the trajectory of every
width toward lower misalignment probability, whereas increasing width shifts
the complete checkpoint trajectory downward. In particular, sufficiently
wide models at early checkpoints can exhibit lower misalignment than much
narrower models after substantially more pretraining. The observed
width-dependent improvement therefore cannot be reduced to training progress
alone.

\paragraph{Effective alignment dimension predicts misalignment.}
The theoretical analysis identifies the gradient mean--covariance geometry,
rather than nominal width itself, as the quantity directly governing
finite-sample directional transfer. We therefore retain every
width--checkpoint--seed configuration as an individual observation and
compare
\begin{equation}
d_{\mathrm{align}}(N;T,s)
\qquad\text{with}\qquad
\widehat{p}_{8,8}(N;T,s),
\label{eq:dalign-phat-comparison}
\end{equation}
where $s$ denotes the training seed. Unlike the heatmap, this analysis does
not average over training seeds and therefore retains all $240$ observations.

As shown in Figure~\ref{fig:extended-misalignment-analysis}(b), larger
$d_{\mathrm{align}}$ is strongly associated with lower empirical
misalignment. Across all widths, checkpoints, and training seeds, the
Spearman rank correlation is
\begin{equation}
\rho_{\mathrm{S}}=-0.858.
\label{eq:overall-spearman-misalignment}
\end{equation}

The five checkpoints obtained from the same trained model are statistically
dependent. We therefore quantify uncertainty by resampling complete model
trajectories indexed by $(N,s)$ rather than treating all $240$ points as
independent. A trajectory-clustered bootstrap with $5{,}000$ replicates gives
the $95\%$ confidence interval
\begin{equation}
[-0.878,-0.771].
\label{eq:cluster-bootstrap-misalignment-ci}
\end{equation}

Importantly, this association is not produced solely by pooling different
pretraining stages. When the Spearman correlation is computed separately at
each checkpoint, it remains strongly negative:
\begin{equation}
\rho_{\mathrm{S}}
=
-0.959,\,-0.953,\,-0.952,\,-0.948,\,-0.944
\label{eq:checkpoint-specific-spearman}
\end{equation}
at $1$, $3$, $5$, $7$, and $9$ billion tokens, respectively. Thus, even
after pretraining progress is held fixed, models with larger effective
alignment dimension consistently exhibit lower finite-sample misalignment.

These results distinguish the proposed mechanism from a direct dependence on
parameter width alone. Model width and pretraining progress alter the
underlying gradient mean and covariance, while $d_{\mathrm{align}}$
summarizes how favorable the resulting geometry is for finite-sample
train--test directional transfer. The empirical evidence therefore supports
the mechanism
\begin{equation}
(N,T)
\longrightarrow
d_{\mathrm{align}}(N;T)
\longrightarrow
\widehat{p}_{M,K}(N;T),
\label{eq:empirical-width-dalign-mechanism}
\end{equation}
rather than merely showing two unrelated correlations with model scale.

Pretraining progress and model width consequently have complementary effects.
For a fixed model width and sampling budget, additional pretraining reduces
the probability of estimating a direction that is misaligned with an
independent test estimate. For a fixed checkpoint and sampling budget,
increasing width shifts the model toward a lower-misalignment regime. Over the
investigated model family, the results support
\begin{equation}
\begin{aligned}
T\uparrow
&\quad\Longrightarrow\quad
\widehat{p}_{M,K}(N;T)\downarrow,
\\
N\uparrow
&\quad\Longrightarrow\quad
\widehat{p}_{M,K}(N;T)\downarrow,
\end{aligned}
\label{eq:empirical-training-width-misalignment}
\end{equation}
for both $(M,K)=(8,8)$ and $(M,K)=(16,16)$ in the representative-width
comparison, and across the complete six-width grid for $M=K=8$.

More directly, the model-level analysis supports
\begin{equation}
d_{\mathrm{align}}(N;T,s)\uparrow
\quad\Longrightarrow\quad
\widehat{p}_{8,8}(N;T,s)\downarrow.
\label{eq:empirical-dalign-misalignment}
\end{equation}
This is precisely the qualitative relationship predicted by the
finite-sample alignment certificate: models with more favorable gradient
mean--covariance geometry are less likely to produce disagreeing training and
test directions from finite samples.

\section{Problem Setup and Notation}
\label{app:problem-setup}

This section introduces the activation-space formulation used throughout
the theoretical analysis. We consider a trained network, a candidate
insertion point, and a function-preserving residual expansion. All
alignment quantities are defined exclusively from the activation-gradient
signal at the selected insertion point.

\paragraph{Data and risks.}
Let \(\mathcal D\) be a probability distribution over
\(\mathcal X\times\mathcal Y\). Let
\[
S_{\mathrm{train}}
=
\{(x_i,y_i)\}_{i=1}^{M}
\sim
\mathcal D^M
\]
and
\[
S_{\mathrm{test}}
=
\{(\widetilde x_j,\widetilde y_j)\}_{j=1}^{K}
\sim
\mathcal D^K
\]
be mutually independent training and test samples, where
\(M,K\in\mathbb N_{+}\).

Let
\[
\ell:
\mathbb R^C\times\mathcal Y
\rightarrow
\mathbb R
\]
be the loss function. For any predictor
\(f:\mathcal X\rightarrow\mathbb R^C\), define
\begin{align*}
L_{\mathrm{train}}(f)
&:=
\frac{1}{M}
\sum_{i=1}^{M}
\ell\bigl(f(x_i),y_i\bigr),\\
L_{\mathrm{test}}(f)
&:=
\frac{1}{K}
\sum_{j=1}^{K}
\ell\bigl(f(\widetilde x_j),\widetilde y_j\bigr),\\
R(f)
&:=
\mathbb E_{(x,y)\sim\mathcal D}
\left[
\ell\bigl(f(x),y\bigr)
\right].
\end{align*}
Here, \(L_{\mathrm{train}}\) and \(L_{\mathrm{test}}\) are empirical
risks evaluated on finite independent samples, whereas \(R\) denotes
the population risk under \(\mathcal D\).

\FloatBarrier

\twocolumn[{%
\begin{@twocolumnfalse}
\centering
\begin{tcolorbox}[
    enhanced,
    width=\textwidth,
    colback=blue!3,
    colframe=blue!45!black,
    boxrule=0.6pt,
    arc=1.5mm,
    left=1.5mm,
    right=1.5mm,
    top=1.2mm,
    bottom=1.2mm,
    title={\textbf{Core Notation}},
    fonttitle=\small,
    before skip=4pt,
    after skip=4pt
]

\footnotesize
\setlength{\tabcolsep}{4pt}
\renewcommand{\arraystretch}{1.28}

\begin{tabularx}{\linewidth}{
    @{}
    >{\raggedright\arraybackslash}p{0.145\linewidth}
    >{\raggedright\arraybackslash}X
    @{}
}

\(\boldsymbol{\mathcal D}\)
&
Population distribution over input--label pairs,
\(\mathcal D\) on \(\mathcal X\times\mathcal Y\).
\\

\(\boldsymbol{M}\)
&
Number of independently sampled training examples in
\(S_{\mathrm{train}}=\{(x_i,y_i)\}_{i=1}^{M}\sim\mathcal D^M\).
\\

\(\boldsymbol{K}\)
&
Number of independently sampled test examples in
\(S_{\mathrm{test}}=\{(\widetilde x_j,\widetilde y_j)\}_{j=1}^{K}
\sim\mathcal D^K\), where \(S_{\mathrm{test}}\) is independent of
\(S_{\mathrm{train}}\).
\\

\(\boldsymbol{f_{\mathrm{bot}}}\)
&
Lower part of the trained network,
\(f_{\mathrm{bot}}:\mathcal X\rightarrow\mathbb R^N\).
\\

\(\boldsymbol{f_{\mathrm{top}}}\)
&
Upper part of the trained network,
\(f_{\mathrm{top}}:\mathbb R^N\rightarrow\mathbb R^C\), with
\(f=f_{\mathrm{top}}\circ f_{\mathrm{bot}}\).
\\

\(\boldsymbol{N}\)
&
Hidden width at the selected residual-module insertion point.
\\

\(\boldsymbol{z(x)}\)
&
Hidden representation at the insertion point,
\(z(x)=f_{\mathrm{bot}}(x)\in\mathbb R^N\).
\\

\(\boldsymbol{q(x,y)}\)
&
Activation gradient at the insertion point,
\(q(x,y)=
\nabla_z\ell(f_{\mathrm{top}}(z),y)
\big|_{z=z(x)}
\in\mathbb R^N\).
\\

\(\boldsymbol{\bar\mu}\)
&
Population mean activation gradient,
\(\bar\mu=
\mathbb E_{(x,y)\sim\mathcal D}[q(x,y)]
\in\mathbb R^N\), with \(\bar\mu\neq0\).
\\

\(\boldsymbol{\Sigma}\)
&
Activation-gradient covariance,
\(\Sigma=
\operatorname{Cov}(q(x,y))
=
\mathbb E[(q(x,y)-\bar\mu)(q(x,y)-\bar\mu)^\top]
\in\mathbb R^{N\times N}\).
\\

\(\boldsymbol{\mu_M}\)
&
Training-set activation-gradient average,
\(\mu_M=
M^{-1}\sum_{i=1}^{M}q(x_i,y_i)\).
\\

\(\boldsymbol{g_K}\)
&
Independent test-set activation-gradient average,
\(g_K=
K^{-1}\sum_{j=1}^{K}q(\widetilde x_j,\widetilde y_j)\).
\\

\(\boldsymbol{d_{\parallel}}\)
&
Directional effective dimension,
\(d_{\parallel}
=
\|\bar\mu\|_2^4/
(\bar\mu^\top\Sigma\bar\mu)\).
\\

\(\boldsymbol{d_2}\)
&
Covariance-energy effective dimension,
\(d_2=
\|\bar\mu\|_2^4/
\|\Sigma\|_F^2
=
\|\bar\mu\|_2^4/
\operatorname{tr}(\Sigma^2)\).
\\

\(\boldsymbol{d_{\mathrm{align}}}\)
&
Effective alignment dimension,
\(d_{\mathrm{align}}
=
\min\{d_{\parallel},d_2\}\).
\\

\(\boldsymbol{n_{\mathrm{eff}}}\)
&
Effective sample size,
\(n_{\mathrm{eff}}
=
MK/(M+K+1)\).
\\

\(\boldsymbol{\mathcal E_{M,K}}\)
&
Activation-gradient misalignment event,
\(\mathcal E_{M,K}
=
\{\mu_M^\top g_K\leq0\}\).
\\

\(\boldsymbol{\widetilde f_{S_{\mathrm{train}},\eta}}\)
&
Training-sample-dependent function-preserving residual-expansion path
satisfying \(\widetilde f_{S_{\mathrm{train}},0}=f\).
\\

\(\boldsymbol{u_{S_{\mathrm{train}}}}\)
&
First-order residual perturbation induced by the expansion path,
\(u_{S_{\mathrm{train}}}(z)
:=
\left.
\frac{\partial}{\partial\eta}
h_{S_{\mathrm{train}},\eta}(z)
\right|_{\eta=0}\).
\\

\end{tabularx}

\end{tcolorbox}
\end{@twocolumnfalse}
}]

\paragraph{Network decomposition.}
Fix a candidate insertion point in a trained network and decompose the
original model as
\[
f
=
f_{\mathrm{top}}
\circ
f_{\mathrm{bot}},
\]
where
\[
f_{\mathrm{bot}}:
\mathcal X\rightarrow\mathbb R^N,
\qquad
f_{\mathrm{top}}:
\mathbb R^N\rightarrow\mathbb R^C.
\]
For an input \(x\in\mathcal X\), let
\[
z(x)
:=
f_{\mathrm{bot}}(x)
\in
\mathbb R^N
\]
denote the hidden representation at the selected insertion point.
Thus,
\[
f(x)
=
f_{\mathrm{top}}\bigl(z(x)\bigr),
\]
where \(N\) is the hidden width at that insertion point.

\paragraph{Function-preserving residual expansion.}
Given the training sample \(S_{\mathrm{train}}\), consider a
differentiable residual-expansion path
\[
\eta
\longmapsto
\widetilde f_{S_{\mathrm{train}},\eta},
\qquad
\eta\geq0,
\]
of the form
\[
\widetilde f_{S_{\mathrm{train}},\eta}(x)
=
f_{\mathrm{top}}
\left(
z(x)
+
h_{S_{\mathrm{train}},\eta}\bigl(z(x)\bigr)
\right),
\]
where
\[
h_{S_{\mathrm{train}},0}(z)=0
\qquad
\text{for every }z\in\mathbb R^N.
\]
Consequently,
\[
\widetilde f_{S_{\mathrm{train}},0}(x)
=
f_{\mathrm{top}}\bigl(z(x)\bigr)
=
f(x),
\]
so the expansion preserves the original input--output function at
\(\eta=0\).

The training sample determines the initial residual direction, while
the scalar \(\eta\) controls the magnitude of the expansion.

\paragraph{Activation-gradient signal.}
For a sample \((x,y)\sim\mathcal D\), define the activation gradient at
the insertion point by
\[
q(x,y)
:=
\nabla_z
\ell\left(
f_{\mathrm{top}}(z),y
\right)
\bigg|_{z=z(x)}
\in
\mathbb R^N.
\]
The random vector \(q(x,y)\) is the loss-gradient signal in the hidden
representation space at the selected insertion point.

Throughout the paper, all population moments, finite-sample averages,
alignment events, and effective dimensions are defined from
\(q(x,y)\).

\paragraph{Population activation-gradient moments.}
Define the population mean activation gradient by
\[
\bar\mu
:=
\mathbb E_{(x,y)\sim\mathcal D}
\left[
q(x,y)
\right]
\in
\mathbb R^N,
\]
and its covariance by
\[
\begin{aligned}
\Sigma
&:=
\operatorname{Cov}\bigl(q(x,y)\bigr)
\in
\mathbb R^{N\times N},
\\
\Sigma
&=
\mathbb E
\left[
\bigl(q(x,y)-\bar\mu\bigr)
\bigl(q(x,y)-\bar\mu\bigr)^\top
\right].
\end{aligned}
\]

The finite-sample alignment analysis assumes
\[
\mathbb E
\left[
\|q(x,y)\|_2^2
\right]
<\infty
\qquad
\text{and}
\qquad
\bar\mu\neq0.
\]
The finite-second-moment condition ensures that both
\(\bar\mu\) and \(\Sigma\) are well defined.

\paragraph{Training and test activation-gradient averages.}
Define the training-set activation-gradient average by
\[
\mu_M
:=
\frac{1}{M}
\sum_{i=1}^{M}
q(x_i,y_i),
\]
and the independently evaluated test-set activation-gradient average by
\[
g_K
:=
\frac{1}{K}
\sum_{j=1}^{K}
q(\widetilde x_j,\widetilde y_j).
\]

Both quantities are unbiased estimators of the same population
activation-gradient signal:
\[
\mathbb E[\mu_M]
=
\mathbb E[g_K]
=
\bar\mu.
\]
Moreover, \(\mu_M\) and \(g_K\) are independent because they are
computed from mutually independent samples.

\paragraph{Activation-gradient alignment event.}
The inner product
\[
\mu_M^\top g_K
\]
measures the directional agreement between the activation gradients
estimated from the training and test samples.

The event
\[
\mu_M^\top g_K>0
\]
corresponds to positive train--test activation-gradient alignment.
The conservative failure event is
\[
\mathcal E_{M,K}
:=
\left\{
\mu_M^\top g_K\leq0
\right\}.
\]

\paragraph{Connection to residual-expansion descent.}

For the training-sample-dependent residual-expansion path
\[
\widetilde f_{S_{\mathrm{train}},\eta}(x)
=
f_{\mathrm{top}}
\left(
z(x)
+
h_{S_{\mathrm{train}},\eta}(z(x))
\right),
\]
define its first-order residual perturbation by
\[
u_{S_{\mathrm{train}}}(z)
:=
\left.
\frac{\partial}{\partial\eta}
h_{S_{\mathrm{train}},\eta}(z)
\right|_{\eta=0}.
\]
Because \(h_{S_{\mathrm{train}},0}(z)=0\), the chain rule gives
\begin{align}
\left.
\frac{\mathrm d}{\mathrm d\eta}
L_{\mathrm{test}}
\bigl(
\widetilde f_{S_{\mathrm{train}},\eta}
\bigr)
\right|_{\eta=0}
&=
\frac{1}{K}
\sum_{j=1}^{K}
\widetilde q_j^\top
u_{S_{\mathrm{train}}}(\widetilde z_j),
\label{eq:general-residual-test-derivative}
\end{align}
where
\[
\widetilde z_j=f_{\mathrm{bot}}(\widetilde x_j),
\qquad
\widetilde q_j
=
q(\widetilde x_j,\widetilde y_j).
\]
Equation~\eqref{eq:general-residual-test-derivative} is the general
first-order derivative for a differentiable function-preserving
residual-expansion path. It does not, by itself, equal
\(-\mu_M^\top g_K\).

\paragraph{Residual-realizability condition.}

Following the direct train--test residual-expansion framework, we assume
that the training-selected residual perturbation satisfies
\begin{equation}
\frac{1}{K}
\sum_{j=1}^{K}
\widetilde q_j^\top
u_{S_{\mathrm{train}}}(\widetilde z_j)
=
-\mu_M^\top g_K.
\label{eq:residual-realizability-condition}
\end{equation}
A sufficient realization is
\[
u_{S_{\mathrm{train}}}(z)
\equiv
-\mu_M,
\]
whenever the residual tangent family contains this constant
activation-space perturbation. Indeed, in this case
\[
\frac{1}{K}
\sum_{j=1}^{K}
\widetilde q_j^\top(-\mu_M)
=
-\mu_M^\top g_K.
\]
Under this condition,
\begin{equation}
\left.
\frac{\mathrm d}{\mathrm d\eta}
L_{\mathrm{test}}
\bigl(
\widetilde f_{S_{\mathrm{train}},\eta}
\bigr)
\right|_{\eta=0}
=
-\mu_M^\top g_K.
\label{eq:alignment-to-test-derivative}
\end{equation}
Therefore,
\[
\mu_M^\top g_K>0
\quad\Longrightarrow\quad
\left.
\frac{\mathrm d}{\mathrm d\eta}
L_{\mathrm{test}}
\bigl(
\widetilde f_{S_{\mathrm{train}},\eta}
\bigr)
\right|_{\eta=0}
<0.
\]
Hence, positive train--test activation-gradient alignment implies that
the training-selected residual perturbation is a strict first-order
descent direction for the independently evaluated test risk.

More precisely, differentiability at \(\eta=0\) gives
\begin{align}
L_{\mathrm{test}}
\bigl(
\widetilde f_{S_{\mathrm{train}},\eta}
\bigr)
&=
L_{\mathrm{test}}(f)
-
\eta\mu_M^\top g_K
+
o(\eta).
\label{eq:test-risk-small-step-expansion}
\end{align}
Equivalently, defining
\[
\Delta_R^{\mathrm{test}}(\eta)
:=
L_{\mathrm{test}}(f)
-
L_{\mathrm{test}}
\bigl(
\widetilde f_{S_{\mathrm{train}},\eta}
\bigr),
\]
we obtain
\[
\Delta_R^{\mathrm{test}}(\eta)
=
\eta\mu_M^\top g_K+o(\eta).
\]
Consequently, if \(\mu_M^\top g_K>0\), then there exists
\(\eta_0>0\) such that
\[
\Delta_R^{\mathrm{test}}(\eta)>0
\qquad
\text{for every }
0<\eta\leq\eta_0.
\]

\paragraph{Effective alignment dimensions.}
Define the directional effective dimension by
\[
d_{\parallel}
:=
\frac{\|\bar\mu\|_2^4}
{\bar\mu^\top\Sigma\bar\mu},
\]
and the covariance-energy effective dimension by
\[
d_2
:=
\frac{\|\bar\mu\|_2^4}
{\|\Sigma\|_F^2}
=
\frac{\|\bar\mu\|_2^4}
{\operatorname{tr}(\Sigma^2)}.
\]
Whenever a denominator is zero, the corresponding effective dimension
is understood to be \(+\infty\).

The quantity \(d_{\parallel}\) measures the squared population
activation-gradient signal relative to covariance in the population
signal direction. The quantity \(d_2\) measures the same signal
relative to the total covariance energy.

Define the effective alignment dimension by
\[
d_{\mathrm{align}}
:=
\min
\left\{
d_{\parallel},
d_2
\right\}.
\]
This provides a conservative scalar summary of the two relative-noise
contributions governing finite-sample activation-gradient alignment.

\paragraph{Effective sample size.}
Define
\[
n_{\mathrm{eff}}
:=
\frac{MK}{M+K+1}.
\]
This quantity combines the training and test sample sizes into the
finite-sample factor appearing in the alignment guarantee. It is
symmetric in \(M\) and \(K\) and increases with either sample size.

In the matched-budget setting \(M=K=n\),
\[
n_{\mathrm{eff}}
=
\frac{n^2}{2n+1}.
\]

\section{Proof of Theorem 1}
\label{app:proof-activation-gradient-alignment}

\begin{proof}
Define the centered empirical activation-gradient errors
\[
\varepsilon_M
:=
\mu_M-\bar\mu,
\qquad
\widetilde{\varepsilon}_K
:=
g_K-\bar\mu.
\]
Because the training and evaluation samples are drawn independently,
\(\varepsilon_M\) and \(\widetilde{\varepsilon}_K\) are independent. Moreover,
\[
\mathbb{E}[\varepsilon_M]
=
\mathbb{E}[\widetilde{\varepsilon}_K]
=
0,
\]
and
\[
\mathbb{E}
\left[
\varepsilon_M\varepsilon_M^\top
\right]
=
\frac{\Sigma}{M},
\qquad
\mathbb{E}
\left[
\widetilde{\varepsilon}_K
\widetilde{\varepsilon}_K^\top
\right]
=
\frac{\Sigma}{K}.
\]

Let
\[
X_{M,K}
:=
\mu_M^\top g_K.
\]
Using
\[
\mu_M=\bar\mu+\varepsilon_M,
\qquad
g_K=\bar\mu+\widetilde{\varepsilon}_K,
\]
we obtain
\[
X_{M,K}
=
\|\bar\mu\|_2^2
+
\bar\mu^\top\varepsilon_M
+
\bar\mu^\top\widetilde{\varepsilon}_K
+
\varepsilon_M^\top\widetilde{\varepsilon}_K.
\]
The three random terms have zero expectation. In particular, independence and centering imply
\[
\mathbb{E}
\left[
\varepsilon_M^\top\widetilde{\varepsilon}_K
\right]
=
\mathbb{E}[\varepsilon_M]^\top
\mathbb{E}[\widetilde{\varepsilon}_K]
=
0.
\]
Therefore,
\[
\mathbb{E}[X_{M,K}]
=
\|\bar\mu\|_2^2.
\]

We next compute the variance. For the first linear term,
\[
\mathbb{E}
\left[
\bigl(\bar\mu^\top\varepsilon_M\bigr)^2
\right]
=
\bar\mu^\top
\mathbb{E}
\left[
\varepsilon_M\varepsilon_M^\top
\right]
\bar\mu
=
\frac{1}{M}
\bar\mu^\top\Sigma\bar\mu.
\]
Similarly,
\[
\mathbb{E}
\left[
\bigl(\bar\mu^\top\widetilde{\varepsilon}_K\bigr)^2
\right]
=
\frac{1}{K}
\bar\mu^\top\Sigma\bar\mu.
\]

For the bilinear term, independence gives
\begin{align}
\mathbb{E}
\left[
\bigl(
\varepsilon_M^\top
\widetilde{\varepsilon}_K
\bigr)^2
\right]
&=
\mathbb{E}
\left[
\operatorname{tr}
\left(
\varepsilon_M\varepsilon_M^\top
\widetilde{\varepsilon}_K
\widetilde{\varepsilon}_K^\top
\right)
\right]
\nonumber\\
&=
\operatorname{tr}
\left(
\mathbb{E}
\left[
\varepsilon_M\varepsilon_M^\top
\right]
\mathbb{E}
\left[
\widetilde{\varepsilon}_K
\widetilde{\varepsilon}_K^\top
\right]
\right)
\nonumber\\
&=
\frac{1}{MK}
\operatorname{tr}(\Sigma^2).
\label{eq:bilinear-second-moment}
\end{align}

All cross terms vanish. First,
\[
\mathbb{E}
\left[
\bigl(\bar\mu^\top\varepsilon_M\bigr)
\bigl(\bar\mu^\top\widetilde{\varepsilon}_K\bigr)
\right]
=
0
\]
by independence and centering. Furthermore,
\begin{align}
\mathbb{E}
\left[
\bigl(\bar\mu^\top\varepsilon_M\bigr)
\bigl(\varepsilon_M^\top\widetilde{\varepsilon}_K\bigr)
\right]
&=
\mathbb{E}_{\varepsilon_M}
\left[
\bigl(\bar\mu^\top\varepsilon_M\bigr)
\varepsilon_M^\top
\mathbb{E}_{\widetilde{\varepsilon}_K}
\left[
\widetilde{\varepsilon}_K
\right]
\right]
\nonumber\\
&=
0.
\end{align}
The analogous cross term involving
\(\bar\mu^\top\widetilde{\varepsilon}_K\) also vanishes. Hence
\[
\operatorname{Var}(X_{M,K})
=
\left(
\frac{1}{M}
+
\frac{1}{K}
\right)
\bar\mu^\top\Sigma\bar\mu
+
\frac{1}{MK}
\operatorname{tr}(\Sigma^2).
\]

Set
\[
m
:=
\mathbb{E}[X_{M,K}]
=
\|\bar\mu\|_2^2
>0,
\qquad
v
:=
\operatorname{Var}(X_{M,K}).
\]
The activation-gradient misalignment event satisfies
\[
\{X_{M,K}\leq0\}
=
\{X_{M,K}-m\leq-m\}.
\]
By Cantelli's one-sided inequality,
\[
\Pr(X_{M,K}\leq0)
\leq
\frac{v}{v+m^2}.
\]
Since \(m^2=\|\bar\mu\|_2^4\), we have
\begin{align}
\frac{v}{m^2}
&=
\left(
\frac{1}{M}
+
\frac{1}{K}
\right)
\frac{\bar\mu^\top\Sigma\bar\mu}
{\|\bar\mu\|_2^4}
+
\frac{1}{MK}
\frac{\operatorname{tr}(\Sigma^2)}
{\|\bar\mu\|_2^4}
\nonumber\\
&=
\frac{1}{M d_{\parallel}}
+
\frac{1}{K d_{\parallel}}
+
\frac{1}{MK d_2}
\nonumber\\
&=
\Xi_{M,K}.
\end{align}
Therefore,
\[
\Pr\!\left(
\mu_M^\top g_K\leq0
\right)
\leq
\frac{\Xi_{M,K}}
{1+\Xi_{M,K}}.
\]

Finally, because
\[
d_{\mathrm{align}}
=
\min\{d_{\parallel},d_2\},
\]
we have
\[
\frac{1}{d_{\parallel}}
\leq
\frac{1}{d_{\mathrm{align}}},
\qquad
\frac{1}{d_2}
\leq
\frac{1}{d_{\mathrm{align}}}.
\]
It follows that
\begin{align}
\Xi_{M,K}
&\leq
\left(
\frac{1}{M}
+
\frac{1}{K}
+
\frac{1}{MK}
\right)
\frac{1}{d_{\mathrm{align}}}
\nonumber\\
&=
\frac{M+K+1}
{MK\,d_{\mathrm{align}}}
\nonumber\\
&=
\frac{1}
{n_{\mathrm{eff}}d_{\mathrm{align}}},
\end{align}
where
\[
n_{\mathrm{eff}}
=
\frac{MK}{M+K+1}.
\]
Since \(x\mapsto x/(1+x)\) is nondecreasing on \([0,\infty)\),
\[
\frac{\Xi_{M,K}}
{1+\Xi_{M,K}}
\leq
\frac{1}
{1+n_{\mathrm{eff}}d_{\mathrm{align}}}.
\]
This completes the proof.
\end{proof}

\begin{corollary}[Balanced sample allocation]
\label{cor:balanced-sample-allocation}
Fix an integer sampling budget \(B\geq2\) and require
\[
M+K=B,
\qquad
M,K\in\mathbb N_{+}.
\]
Then the effective sample size
\[
n_{\mathrm{eff}}(M,K)
=
\frac{MK}{M+K+1}
=
\frac{MK}{B+1}
\]
is maximized when the training and test sample sizes are as balanced as
possible, namely,
\[
|M-K|\leq 1.
\]
More precisely, an optimal allocation is
\[
M^\star
=
\left\lfloor\frac{B}{2}\right\rfloor,
\qquad
K^\star
=
\left\lceil\frac{B}{2}\right\rceil,
\]
or the reversed allocation. The corresponding maximum effective sample
size is
\[
n_{\mathrm{eff}}^\star
=
\frac{
\left\lfloor B^2/4\right\rfloor
}{
B+1
}.
\]
Consequently, for fixed \(B\) and \(d_{\mathrm{align}}\), the simplified
activation-gradient misalignment bound
\[
\Pr\!\left(
\mu_M^\top g_K\leq0
\right)
\leq
\frac{1}{
1+n_{\mathrm{eff}}(M,K)d_{\mathrm{align}}
}
\]
is minimized by a balanced allocation of the sampling budget.
\end{corollary}

\begin{proof}
Because \(M+K=B\) is fixed, maximizing
\(n_{\mathrm{eff}}(M,K)\) is equivalent to maximizing the product
\(MK\). Using
\[
(M-K)^2\geq0,
\]
we obtain
\[
(M+K)^2
\geq
4MK,
\]
and hence
\[
MK
\leq
\frac{B^2}{4}.
\]
For positive integer \(M\) and \(K\), the largest attainable value is
\[
MK
=
\left\lfloor\frac{B^2}{4}\right\rfloor,
\]
which is achieved exactly when \(M\) and \(K\) differ by at most one.
Therefore,
\[
n_{\mathrm{eff}}(M,K)
\leq
\frac{
\left\lfloor B^2/4\right\rfloor
}{
B+1
}.
\]

Finally, the function
\[
x\longmapsto
\frac{1}{1+x\,d_{\mathrm{align}}}
\]
is nonincreasing for \(x\geq0\). Thus, maximizing
\(n_{\mathrm{eff}}(M,K)\) minimizes the corresponding simplified
misalignment-probability bound.
\end{proof}

\section{Proof of Theorem~2}
\label{app:proof-theorem-2}

\begin{proof}
Let \(\widetilde f_S\) denote the empirical jumpboard model in
Theorem~2, and let \(\Delta_R^{\mathrm{test}}\) denote its associated
finite-test margin. These fixed-model quantities are distinct from the
path-indexed margin \(\Delta_R^{\mathrm{test}}(\eta)\) used below in
the existential small-step statement.
Define the two uniform-generalization events
\[
\mathcal{E}_{M}
:=
\left\{
\sup_{f\in\mathcal{H}_{\mathrm{new}}}
\left|
R(f)-L_{\mathrm{train}}(f)
\right|
\leq
\epsilon_M
\right\},
\]
and
\[
\mathcal{E}_{K}
:=
\left\{
\sup_{f\in\mathcal{H}_{\mathrm{new}}}
\left|
R(f)-L_{\mathrm{test}}(f)
\right|
\leq
\epsilon_K
\right\}.
\]
By Assumption~3,
\[
\Pr(\mathcal{E}_M)\geq 1-\delta,
\qquad
\Pr(\mathcal{E}_K)\geq 1-\delta.
\]
Therefore, by the union bound,
\[
\Pr(\mathcal{E}_M\cap\mathcal{E}_K)
\geq
1-2\delta.
\]
No independence between $\mathcal{E}_M$ and $\mathcal{E}_K$ is required.

We now work on the event
$\mathcal{E}_M\cap\mathcal{E}_K$.
Because $f_{\mathrm{new}}\in\mathcal{H}_{\mathrm{new}}$, the test-side and training-side uniform bounds imply
\begin{align}
L_{\mathrm{test}}(f_{\mathrm{new}})
&\leq
R(f_{\mathrm{new}})+\epsilon_K
\nonumber\\
&\leq
L_{\mathrm{train}}(f_{\mathrm{new}})
+\epsilon_M+\epsilon_K.
\label{eq:proof-thm2-new-generalization}
\end{align}
By the definition of the optimization gain,
\[
\Delta_{\mathrm{ERM}}
=
L_{\mathrm{train}}(\widetilde f_S)
-
L_{\mathrm{train}}(f_{\mathrm{new}}),
\]
and hence
\[
L_{\mathrm{train}}(f_{\mathrm{new}})
=
L_{\mathrm{train}}(\widetilde f_S)
-
\Delta_{\mathrm{ERM}}.
\]
Substituting this identity into
\eqref{eq:proof-thm2-new-generalization} gives
\begin{align}
L_{\mathrm{test}}(f_{\mathrm{new}})
&\leq
L_{\mathrm{train}}(\widetilde f_S)
-
\Delta_{\mathrm{ERM}}
+
\epsilon_M+\epsilon_K.
\label{eq:proof-thm2-selection}
\end{align}

Since $\widetilde f_S\in\mathcal{H}_{\mathrm{new}}$, the same uniform events yield
\[
L_{\mathrm{train}}(\widetilde f_S)
\leq
R(\widetilde f_S)+\epsilon_M
\leq
L_{\mathrm{test}}(\widetilde f_S)
+\epsilon_M+\epsilon_K.
\]
Combining this inequality with
\eqref{eq:proof-thm2-selection}, we obtain
\begin{align}
L_{\mathrm{test}}(f_{\mathrm{new}})
&\leq
L_{\mathrm{test}}(\widetilde f_S)
-
\Delta_{\mathrm{ERM}}
+
2(\epsilon_M+\epsilon_K).
\label{eq:proof-thm2-jumpboard}
\end{align}
By the definition of the finite-test jumpboard margin,
\[
\Delta_R^{\mathrm{test}}
=
L_{\mathrm{test}}(f_{\mathrm{old}}^*)
-
L_{\mathrm{test}}(\widetilde f_S),
\]
so that
\[
L_{\mathrm{test}}(\widetilde f_S)
=
L_{\mathrm{test}}(f_{\mathrm{old}}^*)
-
\Delta_R^{\mathrm{test}}.
\]
Substitution into
\eqref{eq:proof-thm2-jumpboard} therefore gives
\[
L_{\mathrm{test}}(f_{\mathrm{new}})
\leq
L_{\mathrm{test}}(f_{\mathrm{old}}^*)
-
\Delta_R^{\mathrm{test}}
-
\Delta_{\mathrm{ERM}}
+
2(\epsilon_M+\epsilon_K).
\]
This inequality holds with probability at least
$1-2\delta$.

Consequently, on the same event, whenever
\[
\Delta_R^{\mathrm{test}}
+
\Delta_{\mathrm{ERM}}
>
2(\epsilon_M+\epsilon_K),
\]
we have
\[
L_{\mathrm{test}}(f_{\mathrm{new}})
<
L_{\mathrm{test}}(f_{\mathrm{old}}^*),
\]
which proves the strict finite-test risk improvement statement.

We next establish the alignment-supported positive-margin guarantee. Define the positive-alignment event
\[
\mathcal{A}_{M,K}
:=
\left\{
\mu_M^\top g_K>0
\right\}.
\]
Define the existential positive-margin event by
\[
\mathcal J_{M,K}
:=
\left\{
\exists\,\eta>0:
\Delta_R^{\mathrm{test}}(\eta)>0
\right\}.
\]
Under the residual-realizability condition
\eqref{eq:residual-realizability-condition}, the expansion
\eqref{eq:test-risk-small-step-expansion} implies
\[
\begin{aligned}
\mu_M^\top g_K>0
\quad\Longrightarrow\quad
&\exists\,\eta_0>0\text{ such that}\\
&\Delta_R^{\mathrm{test}}(\eta)>0
\quad\text{for all }0<\eta\leq\eta_0.
\end{aligned}
\]
Hence,
\[
\mathcal{A}_{M,K}
\subseteq
\mathcal J_{M,K}.
\]
Therefore,
\begin{align}
\Pr(\mathcal J_{M,K})
&\geq
\Pr(\mathcal{A}_{M,K})
\nonumber\\
&=
1-
\Pr\left(
\mu_M^\top g_K\leq 0
\right).
\end{align}
Applying Theorem~1 yields
\[
\Pr(\mathcal J_{M,K})
\geq
1-B_{\mathrm{exact}}(M,K)
\geq
1-B_{\mathrm{align}}(M,K).
\]

Finally, another application of the union bound gives
\begin{align}
&\Pr\left(
\mathcal{E}_M
\cap
\mathcal{E}_K
\cap
\mathcal J_{M,K}
\right)
\nonumber\\
&\qquad\geq
\Pr\left(
\mathcal{E}_M
\cap
\mathcal{E}_K
\cap
\mathcal{A}_{M,K}
\right)
\nonumber\\
&\qquad\geq
1
-
\Pr(\mathcal{E}_M^c)
-
\Pr(\mathcal{E}_K^c)
-
\Pr(\mathcal{A}_{M,K}^c)
\nonumber\\
&\qquad\geq
\left[
1-2\delta-B_{\mathrm{exact}}(M,K)
\right]_+,
\end{align}
where \([x]_+:=\max\{x,0\}\).
Using
$B_{\mathrm{exact}}(M,K)\leq B_{\mathrm{align}}(M,K)$
further gives the simplified lower bound
\[
\Pr\left(
\mathcal{E}_M
\cap
\mathcal{E}_K
\cap
\mathcal J_{M,K}
\right)
\geq
\left[
1-2\delta-B_{\mathrm{align}}(M,K)
\right]_+.
\]

The probability combination again requires no independence between the
alignment event and the two uniform-generalization events. The alignment
certificate guarantees the existence of a positive small-step
finite-test margin; strict improvement of a fixed final expanded model
still requires its realized jumpboard margin and optimization gain to
exceed the generalization penalty.
\end{proof}

\section{Recovery of the Previous Width-Dependent Rate}
\label{app:recovery-width-rate}

The finite-sample certificate in Theorem~1 does not require
any prescribed relation between the activation-gradient
distribution and the model width. We now show that the
stronger covariance-spectrum, total-variance, and signal-growth
conditions used in the previous width-dependent analysis imply
linear growth of the effective alignment dimension and therefore
recover the same-order width-dependent misalignment rate.

\begin{proposition}[Recovery of the width-dependent rate]
\label{prop:recovery-width-rate}
For each insertion width $N$, let
$\bar{\mu}^{(N)}$ and $\Sigma^{(N)}$ denote the population
mean and covariance of the corresponding activation-gradient
signal. Suppose that there exist constants
$c_{\mu}>0$, $C_{\lambda}>0$, $C_{\mathrm{tr}}>0$, and
$N_0\in\mathbb{N}$, all independent of $N$, together with a
positive variance scale $\tau_N^2$, such that for every $N\geq N_0$,
\begin{align}
\left\|\bar{\mu}^{(N)}\right\|_2^2
&\geq c_{\mu}N\tau_N^2,
\label{eq:signal-growth-condition}\\
\lambda_{\max}\!\left(\Sigma^{(N)}\right)
&\leq C_{\lambda}\tau_N^2,
\label{eq:spectral-control-condition}\\
\operatorname{tr}\!\left(\Sigma^{(N)}\right)
&\leq C_{\mathrm{tr}}N\tau_N^2.
\label{eq:total-variance-condition}
\end{align}
Then, for every $N\geq N_0$,
\begin{align}
d_{\parallel}(N)
&\geq
\frac{c_{\mu}}{C_{\lambda}}N,
\label{eq:dparallel-linear-lower-bound}\\
d_2(N)
&\geq
\frac{c_{\mu}^2}
{C_{\lambda}C_{\mathrm{tr}}}N.
\label{eq:d2-linear-lower-bound}
\end{align}
Consequently,
\[
d_{\mathrm{align}}(N)
\geq
c_0N,
\]
where
\[
c_0
:=
\min\left\{
\frac{c_{\mu}}{C_{\lambda}},
\frac{c_{\mu}^2}
{C_{\lambda}C_{\mathrm{tr}}}
\right\}
>0.
\]
Therefore, for every $M,K\in\mathbb{N}_+$,
\begin{align}
\Pr\!\left(
\mu_M^{\top}g_K\leq 0
\right)
&\leq
B_{\mathrm{align}}(M,K;N)
\nonumber\\
&\leq
\frac{1}
{1+c_0n_{\mathrm{eff}}(M,K)N}.
\label{eq:recovered-width-certificate}
\end{align}
In particular,
\[
\Pr\!\left(
\mu_M^{\top}g_K\leq 0
\right)
=
O\!\left(
\frac{1}{MN}
+
\frac{1}{KN}
+
\frac{1}{MKN}
\right).
\]
\end{proposition}

\begin{proof}
We first lower-bound the directional effective dimension.
Since $\Sigma^{(N)}$ is positive semidefinite,
\[
\bar{\mu}^{(N)\top}
\Sigma^{(N)}
\bar{\mu}^{(N)}
\leq
\lambda_{\max}\!\left(\Sigma^{(N)}\right)
\left\|\bar{\mu}^{(N)}\right\|_2^2.
\]
Therefore,
\begin{align}
d_{\parallel}(N)
&=
\frac{
\left\|\bar{\mu}^{(N)}\right\|_2^4
}{
\bar{\mu}^{(N)\top}
\Sigma^{(N)}
\bar{\mu}^{(N)}
}
\nonumber\\
&\geq
\frac{
\left\|\bar{\mu}^{(N)}\right\|_2^2
}{
\lambda_{\max}\!\left(\Sigma^{(N)}\right)
}.
\end{align}
Applying
\eqref{eq:signal-growth-condition} and
\eqref{eq:spectral-control-condition} gives
\[
d_{\parallel}(N)
\geq
\frac{c_{\mu}}{C_{\lambda}}N.
\]

We next lower-bound the covariance-energy effective
dimension. Let
$\lambda_1^{(N)},\ldots,\lambda_N^{(N)}$
denote the eigenvalues of $\Sigma^{(N)}$. Because these
eigenvalues are nonnegative,
\begin{align}
\operatorname{tr}\!\left(
\Sigma^{(N)2}
\right)
&=
\sum_{i=1}^{N}
\left(\lambda_i^{(N)}\right)^2
\nonumber\\
&\leq
\lambda_{\max}\!\left(\Sigma^{(N)}\right)
\sum_{i=1}^{N}\lambda_i^{(N)}
\nonumber\\
&=
\lambda_{\max}\!\left(\Sigma^{(N)}\right)
\operatorname{tr}\!\left(\Sigma^{(N)}\right).
\label{eq:covariance-energy-upper-bound}
\end{align}
Using
\eqref{eq:spectral-control-condition} and
\eqref{eq:total-variance-condition}, we obtain
\[
\operatorname{tr}\!\left(
\Sigma^{(N)2}
\right)
\leq
C_{\lambda}C_{\mathrm{tr}}N\tau_N^4.
\]
Hence,
\begin{align}
d_2(N)
&=
\frac{
\left\|\bar{\mu}^{(N)}\right\|_2^4
}{
\operatorname{tr}\!\left(
\Sigma^{(N)2}
\right)
}
\nonumber\\
&\geq
\frac{
c_{\mu}^2N^2\tau_N^4
}{
C_{\lambda}C_{\mathrm{tr}}N\tau_N^4
}
\nonumber\\
&=
\frac{
c_{\mu}^2
}{
C_{\lambda}C_{\mathrm{tr}}
}N.
\end{align}

Because
\[
d_{\mathrm{align}}(N)
=
\min\left\{
d_{\parallel}(N),d_2(N)
\right\},
\]
the two preceding bounds imply
\[
d_{\mathrm{align}}(N)
\geq
c_0N,
\]
with
\[
c_0
=
\min\left\{
\frac{c_{\mu}}{C_{\lambda}},
\frac{c_{\mu}^2}
{C_{\lambda}C_{\mathrm{tr}}}
\right\}.
\]

Applying the simplified certificate from Theorem~1 yields
\begin{align}
\Pr\!\left(
\mu_M^{\top}g_K\leq 0
\right)
&\leq
\frac{
1
}{
1+n_{\mathrm{eff}}(M,K)
d_{\mathrm{align}}(N)
}
\nonumber\\
&\leq
\frac{
1
}{
1+c_0n_{\mathrm{eff}}(M,K)N
}.
\end{align}
Since
\[
n_{\mathrm{eff}}(M,K)
=
\frac{MK}{M+K+1},
\]
we have
\begin{align}
\frac{1}
{c_0n_{\mathrm{eff}}(M,K)N}
&=
\frac{M+K+1}
{c_0MKN}
\nonumber\\
&=
\frac{1}{c_0KN}
+
\frac{1}{c_0MN}
+
\frac{1}{c_0MKN}.
\end{align}
Therefore,
\[
\Pr\!\left(
\mu_M^{\top}g_K\leq 0
\right)
=
O\!\left(
\frac{1}{MN}
+
\frac{1}{KN}
+
\frac{1}{MKN}
\right).
\]
This completes the proof.
\end{proof}

The same rate can also be recovered directly from the exact
certificate. Under the conditions of
Proposition~\ref{prop:recovery-width-rate},
\begin{align}
\Xi_{M,K}(N)
&=
\frac{1}{Md_{\parallel}(N)}
+
\frac{1}{Kd_{\parallel}(N)}
+
\frac{1}{MKd_2(N)}
\nonumber\\
&\leq
\frac{C_{\lambda}}{c_{\mu}N}
\left(
\frac{1}{M}+\frac{1}{K}
\right)
+
\frac{
C_{\lambda}C_{\mathrm{tr}}
}{
c_{\mu}^2MKN
}.
\end{align}
Since
\[
B_{\mathrm{exact}}(M,K;N)
=
\frac{\Xi_{M,K}(N)}
{1+\Xi_{M,K}(N)}
\leq
\Xi_{M,K}(N),
\]
the exact certificate satisfies the same asymptotic rate.

These conditions are sufficient rather than necessary.
Theorem~1 remains valid even when the signal norm, covariance
spectrum, or total variance does not obey the prescribed
width-growth relations. Their role is only to provide one
particular route to
$d_{\mathrm{align}}(N)=\Omega(N)$ and thereby recover the
previous width-dependent guarantee.

\section{Exact Parameter-Space Realization of the Activation-Gradient Direction}
\label{app:weight-gradient-realization}

We now show that the activation-gradient direction used in the
finite-sample analysis can be realized exactly by a concrete
function-preserving residual module. All quantities $z$, $q$, $\mu_M$,
$g_K$, $L_{\mathrm{train}}$, and $L_{\mathrm{test}}$ follow the notation
introduced above.

Consider an inserted residual module of the form
\[
h_{U,V,b}(z)
=
V\psi_U(z)+b,
\]
where $V\in\mathbb{R}^{N\times m}$ is the output projection and
$b\in\mathbb{R}^{N}$ is an additive output parameter. Initialize the new
module at
\[
U=U_0,
\qquad
V=0,
\qquad
b=0.
\]
Because
\[
h_{U_0,0,0}(z)=0
\qquad
\text{for every }z,
\]
the insertion preserves the original network function exactly.

At this function-preserving initialization, the sample-level gradients with
respect to the new output parameters satisfy the following identities, where
$\widetilde z(V,b):=z+V\psi_{U_0}(z)+b$:
\begin{align*}
\left.
\nabla_V
\ell\!\left(
f_{\mathrm{top}}
\bigl(\widetilde z(V,b)\bigr),
y
\right)
\right|_{V=0,b=0}
&=
q(x,y)\psi_{U_0}(z)^\top,
\\
\left.
\nabla_b
\ell\!\left(
f_{\mathrm{top}}
\bigl(\widetilde z(V,b)\bigr),
y
\right)
\right|_{V=0,b=0}
&=
q(x,y).
\end{align*}
Thus, $q(x,y)$ is the prediction-side factor of the output-projection
weight gradient and, more importantly, is exactly the sample-level parameter
gradient with respect to the additive output parameter $b$.

\begin{proposition}[Exact realization of the activation-gradient direction]
\label{prop:exact-activation-gradient-realization}
Suppose that the inserted residual module contains the zero-initialized
additive output parameter $b$. Then the training-selected direction
$-\mu_M$ is realized by the parameter path
\[
U(\eta)=U_0,
\qquad
V(\eta)=0,
\qquad
b(\eta)=-\eta\mu_M.
\]
The induced residual perturbation satisfies
\[
u_{S_{\mathrm{train}}}(z)
=
-\mu_M
\qquad
\text{for every }z,
\]
and the directional derivative of the independent finite-test risk is
\begin{equation}
\left.
\frac{d}{d\eta}
L_{\mathrm{test}}
\bigl(
\widetilde f_{S_{\mathrm{train}},\eta}
\bigr)
\right|_{\eta=0}
=
-\mu_M^\top g_K.
\label{eq:exact-realized-test-derivative}
\end{equation}
Consequently, positive train--test activation-gradient alignment implies
that a sufficiently small realizable update of the inserted residual module
strictly decreases the independent finite-test risk.
\end{proposition}

\begin{proof}
Along the stated parameter path,
\[
h_\eta(z)
=
-\eta\mu_M,
\]
and hence
\[
u_{S_{\mathrm{train}}}(z)
=
\left.
\frac{\partial}{\partial\eta}
h_\eta(z)
\right|_{\eta=0}
=
-\mu_M.
\]
Therefore, the residual-realizability condition
\eqref{eq:residual-realizability-condition} holds exactly.

Applying the chain rule to the empirical training risk gives
\[
\left.
\frac{d}{d\eta}
L_{\mathrm{train}}
\bigl(
\widetilde f_{S_{\mathrm{train}},\eta}
\bigr)
\right|_{\eta=0}
=
-\|\mu_M\|_2^2.
\]
Thus, whenever $\mu_M\neq0$, the path is a strict first-order descent
direction for the empirical training risk.

For the independent test sample, the same calculation gives
\begin{align*}
\left.
\frac{d}{d\eta}
L_{\mathrm{test}}
\bigl(
\widetilde f_{S_{\mathrm{train}},\eta}
\bigr)
\right|_{\eta=0}
&=
\frac{1}{K}
\sum_{j=1}^{K}
\widetilde q_j^\top(-\mu_M)
\\
&=
-\mu_M^\top g_K,
\end{align*}
which proves \eqref{eq:exact-realized-test-derivative}.

If $\mu_M^\top g_K>0$, the derivative is strictly negative. By
differentiability at $\eta=0$,
\[
L_{\mathrm{test}}
\bigl(
\widetilde f_{S_{\mathrm{train}},\eta}
\bigr)
=
L_{\mathrm{test}}(f)
-
\eta\mu_M^\top g_K
+
o(\eta).
\]
Hence, there exists $\eta_0>0$ such that
\[
L_{\mathrm{test}}
\bigl(
\widetilde f_{S_{\mathrm{train}},\eta}
\bigr)
<
L_{\mathrm{test}}(f)
\]
for every $0<\eta\leq\eta_0$.
\end{proof}

The additive output parameter has an equivalent constant-channel
representation. Specifically, define
\[
\widetilde\psi_U(z)
=
\begin{bmatrix}
\psi_U(z)\\
1
\end{bmatrix},
\qquad
\widetilde V
=
\begin{bmatrix}
V & b
\end{bmatrix}.
\]
Then
\[
V\psi_U(z)+b
=
\widetilde V\widetilde\psi_U(z).
\]
Thus, an affine output projection, a zero-initialized output bias, and a
constant residual channel provide equivalent realizations of the same
activation-space tangent direction. This requirement applies only to the
newly inserted module and does not require the pretrained base architecture
to use biased linear layers.

If the inserted module contains neither an additive output parameter nor a
constant feature, $q(x,y)$ alone does not generally determine a realizable
parameter-space update. In that case, the appropriate sample-level signal is
the full output-projection gradient
\[
\zeta(x,y)
=
\operatorname{vec}
\left(
q(x,y)\psi_{U_0}(z)^\top
\right),
\]
and the finite-sample alignment analysis must instead be formulated in terms
of $\zeta$.

\section{Architecture-Specific Selection of Activation-Gradient Measurement Locations}
\label{app:architecture-specific-measurement-locations}

The normalized residual-expansion framework is formulated at a designated intermediate representation and therefore permits architecture-specific measurement locations. The selected location must provide a valid residual insertion point, preserve the function-preserving expansion interpretation, and allow the activation gradient to represent the relevant first-order improvement signal. Based on these principles, we use the final residual representation for Transformer language models and an early residual block for ResNet-20.

\paragraph{Existential nature of the layer-wise theory.}
Both the residual-expansion condition and our finite-sample alignment
certificate are existential over admissible insertion locations. Let
$\mathcal{L}_{\mathrm{adm}}$ denote the set of admissible insertion layers.
For each $\ell\in\mathcal{L}_{\mathrm{adm}}$, the corresponding gradient
signal, mean, covariance, effective alignment dimension, and misalignment
probability are defined locally at that layer.

It is sufficient that there exists one layer
$\ell^\star\in\mathcal{L}_{\mathrm{adm}}$ such that
\begin{equation}
\mathcal{C}_{\ell^\star}\ \text{holds}
\qquad\text{and}\qquad
B_{M,K}^{(\ell^\star)}\leq\delta,
\label{eq:existential-layer-condition}
\end{equation}
where $\mathcal{C}_{\ell}$ denotes the population expansion condition at
layer $\ell$, and $B_{M,K}^{(\ell)}$ is the corresponding finite-sample upper
bound on the gradient-misalignment probability. It then follows that
\begin{equation}
\Pr\!\left[
\left(\mu_M^{(\ell^\star)}\right)^\top
g_K^{(\ell^\star)}>0
\right]
\geq 1-\delta.
\label{eq:existential-layer-alignment}
\end{equation}
Hence, with probability at least $1-\delta$, a function-preserving block
inserted at $\ell^\star$ admits a train-estimated direction aligned with the
corresponding test direction. Together with
$\mathcal{C}_{\ell^\star}$, a sufficiently small update yields a nearby
expanded model with lower population risk.

Therefore, neither the original residual-expansion argument nor our
finite-sample refinement requires the condition to hold at every admissible
layer. A single valid insertion location is sufficient to establish the
existence of an improving structural expansion. By contrast, ruling out
improvement through the considered block family and first-order expansion
mechanism requires the corresponding condition to fail at every admissible
insertion layer.

\subsection{Why the Final Residual Representation Is Measured in LLMs}
\label{app:llm-final-representation}

The original normalized residual-expansion framework is defined at an arbitrary designated intermediate representation. Specifically, a trained network can be decomposed as \(f=f_{\mathrm{top}}\circ f_{\mathrm{bot}}\), where \(z=f_{\mathrm{bot}}(x)\in\mathbb R^N\) denotes the residual-stream representation at the selected insertion point and \(q(x,y)=\nabla_z\ell(f_{\mathrm{top}}(z),y)\) denotes the corresponding activation-gradient signal. The theory therefore does not prescribe a unique layer shared by all architectures. Instead, the selected location must constitute a valid residual insertion point and must preserve the interpretation of \(q\) as the first-order signal governing a function-preserving residual expansion.

Consider an \(L\)-block Transformer with residual-stream representations \(z_0,z_1,\ldots,z_L\), where
\[
z_{l+1}=T_l(z_l),\qquad l=0,\ldots,L-1,
\]
and let \(H\) denote the final normalization and language-modeling head. For each residual-stream location \(l\), define
\[
q_l(x,y)
=
\nabla_{z_l}
\ell\!\left(H(z_L),y\right).
\]
Let
\[
\Phi_{l\rightarrow L}
=
T_{L-1}\circ\cdots\circ T_l
\]
denote the downstream Transformer mapping from \(z_l\) to \(z_L\), and define its Jacobian by
\[
A_l(x)
=
J_{\Phi_{l\rightarrow L}}(z_l).
\]
The chain rule then gives
\[
q_l(x,y)
=
A_l(x)^\top q_L(x,y),
\]
where
\[
q_L(x,y)
=
\nabla_{z_L}
\ell\!\left(H(z_L),y\right).
\]
Thus, an activation gradient measured at an earlier layer is obtained by transporting the terminal activation gradient through every remaining Transformer block.

This transport directly affects the population quantities entering the effective alignment dimension. At layer \(l\), the population mean and covariance are
\[
\bar{\mu}_l
=
\mathbb E
\left[
A_l(x)^\top q_L(x,y)
\right]
\]
and
\[
\Sigma_l
=
\operatorname{Cov}
\left(
A_l(x)^\top q_L(x,y)
\right).
\]
Consequently, \(\bar{\mu}_l\) and \(\Sigma_l\) depend not only on the terminal loss-gradient geometry, but also on the anisotropy, conditioning, sample dependence, and gradient correlation induced by the downstream Jacobian \(A_l(x)\). The corresponding effective dimensions
\[
d_{\parallel}^{(l)}
=
\frac{\|\bar{\mu}_l\|_2^4}
{\bar{\mu}_l^\top\Sigma_l\bar{\mu}_l},
\qquad
d_2^{(l)}
=
\frac{\|\bar{\mu}_l\|_2^4}
{\operatorname{tr}(\Sigma_l^2)}
\]
therefore combine the intrinsic activation-gradient geometry at the prediction interface with an additional layer-dependent transport effect.

The distinction can be made explicit by first considering an idealized
deterministic downstream Jacobian of the form
\[
A_l=cR,
\]
where $c>0$ and $R\in\mathbb R^{N\times N}$ is orthogonal:
\[
R^\top R=RR^\top=I.
\]
In this case,
\[
\bar{\mu}_l
=
cR^\top\bar{\mu}_L,
\qquad
\Sigma_l
=
c^2R^\top\Sigma_LR.
\]
Orthogonal invariance then gives
\[
\|\bar{\mu}_l\|_2^4
=
c^4\|\bar{\mu}_L\|_2^4,
\]
\[
\bar{\mu}_l^\top\Sigma_l\bar{\mu}_l
=
c^4\bar{\mu}_L^\top\Sigma_L\bar{\mu}_L,
\]
and
\[
\operatorname{tr}(\Sigma_l^2)
=
c^4\operatorname{tr}(\Sigma_L^2).
\]
It follows that
\[
d_{\parallel}^{(l)}
=
d_{\parallel}^{(L)},
\qquad
d_2^{(l)}
=
d_2^{(L)},
\qquad
d_{\mathrm{align}}^{(l)}
=
d_{\mathrm{align}}^{(L)}.
\]
Hence, a deterministic scaled isometry preserves the effective alignment geometry exactly. Differences between early- and late-layer measurements arise from deviations from this ideal case, including unequal singular values, sample-dependent Jacobians, and correlations between the downstream transformation and the terminal gradient.

To isolate the sample-dependent contribution, write
\[
A_l(x)
=
\overline A_l+\Delta A_l(x),
\qquad
\overline A_l
=
\mathbb E[A_l(x)].
\]
The earlier-layer gradient can then be decomposed as
\[
q_l
=
\overline A_l^\top q_L+r_l,
\qquad
r_l
=
\Delta A_l(x)^\top q_L.
\]
The additional transport term satisfies
\[
\mathbb E\|r_l\|_2^2
\le
\mathbb E
\left[
\|\Delta A_l(x)\|_{\mathrm{op}}^2
\|q_L(x,y)\|_2^2
\right].
\]
Accordingly, input-dependent variation in the downstream Transformer Jacobian introduces an additional source of second-moment variation that can modify both the estimated signal direction and its covariance. Moreover, because
\[
A_l(x)
=
J_{T_{L-1}}(z_{L-1})
\cdots
J_{T_l}(z_l),
\]
earlier residual representations involve a longer product of attention- and feed-forward-block Jacobians, providing more opportunities for anisotropic scaling, rotation, and sample-dependent distortion.

At the output of the final Transformer block, the downstream Transformer mapping is the identity:
\[
A_L=I.
\]
Therefore,
\[
q_L
=
\nabla_{z_L}
\ell\!\left(H(z_L),y\right),
\]
and the transport perturbation satisfies
\[
\Delta A_L=0,
\qquad
r_L=0.
\]
The final residual representation consequently removes all cumulative Jacobian transport through additional Transformer blocks. Its mean \(\bar{\mu}_L\), covariance \(\Sigma_L\), and effective alignment dimension \(d_{\mathrm{align}}^{(L)}\) provide the most direct characterization of the activation-gradient geometry connected to the language-modeling loss.

This location remains fully compatible with the residual-expansion mechanism of the original theory. A Transformer MLP branch can be written as
\[
h_{U,V}(z)
=
W_{\mathrm{out}}\psi_U(z),
\]
while an attention branch can be written as
\[
h_{U,V}(z)
=
W_O\operatorname{Attn}_U(z).
\]
In both cases, the output projection plays the role of \(V\) in the abstract residual form \(h_{U,V}(z)=V\psi_U(z)\). Initializing this output projection at zero gives \(h_{U,0}(z)=0\), preserving the original model function while retaining a realizable first-order residual direction. The final residual stream is therefore an admissible insertion point under the same function-preserving residual-expansion framework used by the theory.

For the LLaMA-style and Pythia models, we consequently measure the activation gradient at the output of the final Transformer block and before the final normalization and language-modeling head. This choice follows a minimal-transport principle: among residual-stream locations, it eliminates the largest amount of downstream Jacobian-induced variation and provides the least confounded measurement of width-dependent activation-gradient geometry. We do not claim that the final block universally maximizes \(d_{\mathrm{align}}\) over all possible layers; rather, it is the canonical location at which the measured statistics are not additionally transformed by any subsequent Transformer block.

\subsection{Why an Early Residual Block Is Measured in ResNet}
\label{app:resnet-early-representation}

The original normalized residual-expansion framework permits the insertion point to be placed at an arbitrary designated intermediate representation. For a ResNet with intermediate representations \(z_0,z_1,\ldots,z_L\), write
\[
z_{j+1}=T_j(z_j),
\qquad
j=0,\ldots,L-1,
\]
where \(T_j\) denotes the \(j\)-th residual block, including the skip connection and the residual branch. At a candidate insertion point \(l\), define the activation gradient
\[
q_l(x,y)
=
\nabla_{z_l}
\ell\!\left(f(x),y\right).
\]
A function-preserving residual expansion introduces a new branch \(h_{l,\eta}\) satisfying \(h_{l,0}(z)=0\). Its first-order perturbation is
\[
u_l(z_l)
=
\left.
\frac{\partial}{\partial\eta}
h_{l,\eta}(z_l)
\right|_{\eta=0}.
\]
The corresponding population-risk derivative is
\[
\left.
\frac{\mathrm d}{\mathrm d\eta}
R\!\left(f_{l,\eta}\right)
\right|_{\eta=0}
=
\mathbb E
\left[
q_l(x,y)^\top u_l(z_l)
\right].
\]
Therefore, the suitability of an insertion point depends on whether the realizable residual tangent family contains a direction that is non-orthogonal to the activation-gradient signal.

Let \(\mathcal U_l\) denote the set of first-order perturbations realizable by a zero-output residual block at location \(l\):
\[
\mathcal U_l
=
\left\{
u_l:
u_l(z_l)
=
\left.
\frac{\partial}{\partial\eta}
h_{l,\eta}(z_l)
\right|_{\eta=0}
\right\}.
\]
We assume that each residual tangent family is closed under
nonnegative scalar multiplication:
\[
u\in\mathcal U_l,\quad a\geq0
\quad\Longrightarrow\quad
au\in\mathcal U_l.
\]
This property holds when the tangent family is generated by
differentiable parameter paths whose initial velocities may be
rescaled.
Define the normalized first-order improvement capacity at layer \(l\) by
\[
\Gamma_l
=
\sup_{\substack{u_l\in\mathcal U_l\\
\|u_l\|_{L^2(\mathcal D)}\le1}}
-
\mathbb E
\left[
q_l^\top u_l
\right].
\]
A positive value \(\Gamma_l>0\) means that the residual module can realize a strict population-risk descent direction at that location.

Consider an early location \(l\) and a later location \(r>l\). Let
\[
\Phi_{l\rightarrow r}
=
T_{r-1}\circ\cdots\circ T_l
\]
be the downstream mapping from \(z_l\) to \(z_r\), and define
\[
A_{l\rightarrow r}(x)
=
J_{\Phi_{l\rightarrow r}}(z_l(x)).
\]
By the chain rule,
\[
q_l(x,y)
=
A_{l\rightarrow r}(x)^\top q_r(x,y).
\]
The following proposition formalizes when an early residual block is at least as capable of producing a first-order improvement as a later block.

\begin{proposition}[Early-layer dominance under residual-tangent coverage]
\label{prop:early-layer-tangent-coverage}
Let \(l<r\). Suppose that there exists \(\kappa_{l,r}>0\) such that, for every \(u_r\in\mathcal U_r\), there exists \(u_l\in\mathcal U_l\) satisfying
\[
A_{l\rightarrow r}(x)u_l(z_l(x))
=
u_r(z_r(x))
\]
almost surely, and
\[
\|u_l\|_{L^2(\mathcal D)}
\le
\kappa_{l,r}^{-1}
\|u_r\|_{L^2(\mathcal D)}.
\]
Then
\[
\Gamma_l
\ge
\kappa_{l,r}\Gamma_r.
\]
In particular, if the later location admits a strict first-order descent direction, then the earlier location also admits a strict first-order descent direction.
\end{proposition}

\begin{proof}
Take any \(u_r\in\mathcal U_r\) satisfying
\[
\|u_r\|_{L^2(\mathcal D)}
\le1.
\]
By the residual-tangent coverage assumption, there exists \(u_l\in\mathcal U_l\) such that
\[
A_{l\rightarrow r}u_l=u_r,
\qquad
\|u_l\|_{L^2(\mathcal D)}
\le
\kappa_{l,r}^{-1}.
\]
Define
\[
\widetilde u_l
=
\kappa_{l,r}u_l.
\]
The nonnegative scaling closure ensures that
\(\widetilde u_l\in\mathcal U_l\). Moreover,
\(\|\widetilde u_l\|_{L^2(\mathcal D)}\le1\), and the chain-rule identity gives
\[
-\mathbb E
\left[
q_l^\top\widetilde u_l
\right]
=
-\kappa_{l,r}
\mathbb E
\left[
q_r^\top
A_{l\rightarrow r}u_l
\right].
\]
Since \(A_{l\rightarrow r}u_l=u_r\), it follows that
\[
-\mathbb E
\left[
q_l^\top\widetilde u_l
\right]
=
-\kappa_{l,r}
\mathbb E
\left[
q_r^\top u_r
\right].
\]
Taking the supremum over all admissible \(u_r\) yields
\[
\Gamma_l
\ge
\kappa_{l,r}\Gamma_r.
\]
Therefore, any first-order descent direction realizable at the later location can also be realized at the earlier location, up to the conditioning factor \(\kappa_{l,r}\).
\end{proof}

Proposition~\ref{prop:early-layer-tangent-coverage} provides the main representational reason for selecting an early ResNet block. Early ResNet representations retain a high spatial resolution and a comparatively large number of local degrees of freedom. A standard BasicBlock residual branch can be expressed as
\[
h_{U,V}(z)
=
\operatorname{Conv}_{\mathrm{out}}
\left(
\psi_U(z)
\right),
\]
which is an instance of the abstract realizability form
\[
h_{U,V}(z)
=
V\psi_U(z).
\]
The feature map \(\psi_U(z)\) contains the preceding convolutions, normalization operations, and nonlinearities, while the final convolutional projection plays the role of \(V\). At an early stage, this residual branch acts on a high-resolution feature map and therefore generates a broad family of spatially structured tangent directions.

As the representation passes through later stages, stride-two downsampling and channel transformations compress and reorganize the spatial degrees of freedom. Under the tangent-coverage condition in Proposition~\ref{prop:early-layer-tangent-coverage}, the image of the early residual tangent family under the downstream Jacobian contains the tangent family available at the later location:
\[
\mathcal U_r
\subseteq
A_{l\rightarrow r}\mathcal U_l.
\]
Thus, the early module can reproduce any perturbation available to the later module after downstream propagation, while also retaining additional perturbation directions that may be removed by spatial downsampling or later feature compression. This establishes a precise sense in which the early residual location is representationally at least as expressive for first-order residual expansion.

The finite-sample alignment geometry gives a second justification. At layer \(l\), define
\[
\bar\mu_l
=
\mathbb E[q_l],
\qquad
\Sigma_l
=
\operatorname{Cov}(q_l),
\]
and
\[
d_{\parallel}^{(l)}
=
\frac{\|\bar\mu_l\|_2^4}
{\bar\mu_l^\top\Sigma_l\bar\mu_l},
\qquad
d_2^{(l)}
=
\frac{\|\bar\mu_l\|_2^4}
{\operatorname{tr}(\Sigma_l^2)}.
\]
Suppose that the candidate ResNet locations satisfy the uniform covariance controls
\[
\lambda_{\max}(\Sigma_l)
\le
\Lambda,
\qquad
\operatorname{tr}(\Sigma_l^2)
\le
T,
\]
where \(\Lambda\) and \(T\) do not increase across the compared locations. Then
\[
\bar\mu_l^\top\Sigma_l\bar\mu_l
\le
\Lambda\|\bar\mu_l\|_2^2,
\]
and therefore
\[
d_{\parallel}^{(l)}
\ge
\frac{\|\bar\mu_l\|_2^2}{\Lambda}.
\]
Similarly,
\[
d_2^{(l)}
\ge
\frac{\|\bar\mu_l\|_2^4}{T}.
\]
Consequently,
\[
d_{\mathrm{align}}^{(l)}
\ge
\min
\left\{
\frac{\|\bar\mu_l\|_2^2}{\Lambda},
\frac{\|\bar\mu_l\|_2^4}{T}
\right\}.
\]

The original residual-expansion analysis predicts that the remaining first-order signal becomes weaker as the model approaches the deepest-model regime. If an early candidate location \(l\) and a later location \(r\) satisfy
\[
\|\bar\mu_l\|_2
\ge
\|\bar\mu_r\|_2,
\]
then the preceding lower bound gives
\[
\min
\left\{
\frac{\|\bar\mu_l\|_2^2}{\Lambda},
\frac{\|\bar\mu_l\|_2^4}{T}
\right\}
\ge
\min
\left\{
\frac{\|\bar\mu_r\|_2^2}{\Lambda},
\frac{\|\bar\mu_r\|_2^4}{T}
\right\}.
\]
Thus, under comparable covariance control, the early layer has a no-weaker guaranteed lower certificate for the effective alignment dimension. The corresponding finite-sample misalignment guarantee satisfies
\[
\Pr
\left(
\mu_M^{(l)\top}g_K^{(l)}
\le0
\right)
\le
\frac{1}
{1+n_{\mathrm{eff}}d_{\mathrm{align}}^{(l)}},
\]
so a stronger early-layer signal leads to a tighter guaranteed alignment certificate.

The representational and statistical arguments are complementary. The tangent-coverage result shows that an early residual block can reproduce later first-order perturbations through the remaining network, while the signal-to-noise argument shows that an early location is preferable when the population first-order signal has not yet entered the weak-signal regime. This is consistent with the deepest-model mechanism of the original analysis, according to which the remaining activation-gradient signal becomes progressively weaker as residual depth is exhausted.

For ResNet-20, we therefore measure the activation gradient at the output of the first residual block. This location retains the full spatial resolution of the first residual stage, provides a rich realizable residual tangent family, and leaves the subsequent residual stages available to transform the inserted perturbation into a prediction-level change. The choice also follows the early-block regime used in the original ResNet covariance experiment. It should not be interpreted as an unconditional theorem that the first residual block maximizes \(d_{\mathrm{align}}\) for every trained ResNet. Rather, it is the architecture-specific location favored by residual-tangent coverage, non-exhausted first-order signal, and controlled within-family width comparison.

\finishappendixdocument
\end{document}